\documentclass{article}

\usepackage[letterpaper,top=2cm,bottom=2cm,left=3cm,right=3cm,marginparwidth=2cm]{geometry}

\usepackage{microtype}
\usepackage{graphicx}
\usepackage{subfigure}
\usepackage{booktabs} % for professional tables

\usepackage{hyperref}

\usepackage{todonotes}

\renewcommand{\epsilon}{\varepsilon}
\newcommand{\eps}{\varepsilon}

\usepackage{amsmath}

\usepackage{comment}
\usepackage{amsthm}
\usepackage{amssymb}
\usepackage{amsfonts}
\usepackage{dsfont}
\usepackage{thm-restate}
\usepackage{nicefrac}
\usepackage{footnote}
\usepackage[capitalize,noabbrev]{cleveref}
\usetikzlibrary{patterns}

\usepackage{booktabs}
\usepackage{algorithm}
\usepackage{algorithmic}

\newcommand{\calX}{\mathcal{X}}
\newcommand{\calD}{\mathcal{D}}
\newcommand{\calC}{\mathcal{C}}
\newcommand{\calA}{\mathcal{A}}
\newcommand{\cost}{\Phi}

\newcommand{\disp}{\text{Disp}}
\newcommand{\R}{\mathbb{R}}
\newcommand{\E}{\mathbb{E}}
\usepackage{natbib} 
\DeclareMathOperator{\argmin}{argmin}

\theoremstyle{plain}
\newtheorem{thm}{Theorem} %[section]
\newtheorem*{thm*}{Theorem}

\newtheorem{definition}[thm]{Definition}
\newtheorem{rem}[thm]{Remark}

\newtheorem{assumption}[thm]{Assumption}

\newtheorem{lem}[thm]{Lemma}

\usepackage{wrapfig}
\newcommand{\ERCLogo}{%
\begin{wrapfigure}{r}{2.1cm}
        \centering
        \vspace{-0.5cm}
        \includegraphics[width=2cm]{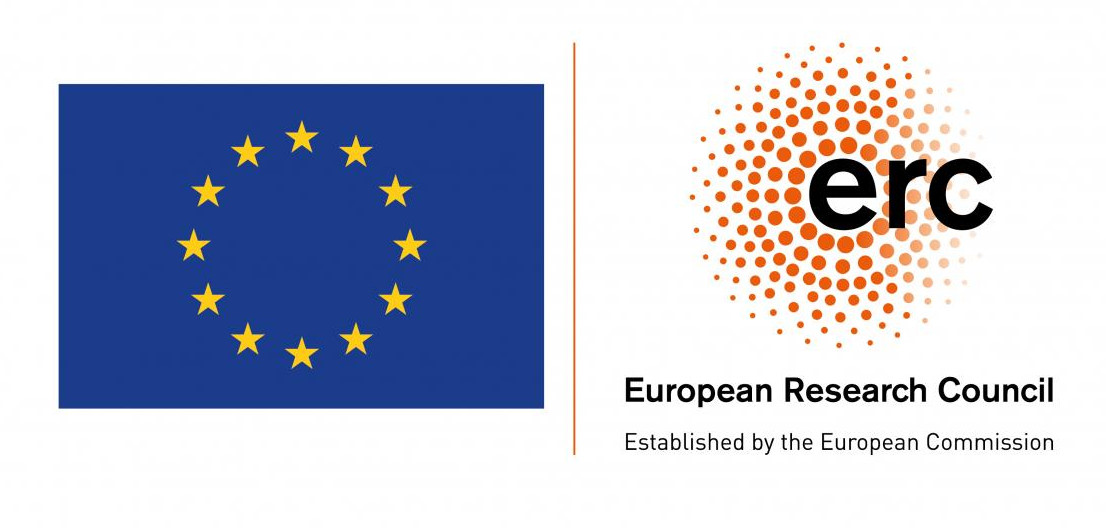}
\end{wrapfigure}
}

\date{}
\author{
Kyriakos Axiotis \\ Google Research \and {Vincent Cohen-Addad} \\ Google Research
\and  Monika Henzinger\footnote{ 
  Funded by the  European Research Council under the European Union's Horizon 2020 research and innovation  programme, ERC grant no.~788183, ``Alpha Shape Theory Extended (Alpha)'',
  by the Wittgenstein Prize, FWF grant no.~Z~342-N31, 
  and by the DFG Collaborative Research Center TRR~109, 
  FWF grant no.~I~02979-N35.
  M.Henzinger %\flagpic{LOGO_ERC-FLAG_EU_.jpg}
  received funding by the European Research Council under the European Union's Horizon 2020 research and innovation programme, ERC grant no.~101019564,   \ERCLogo
  ``The Design of Modern Fully Dynamic Data Structures (MoDynStruct)'',
  and by the Austrian Science Fund through the Wittgenstein Prize with FWF grant no.\ Z 422-N, and also by FWF grant no.~I~5982-N, and
  by FWF grant no.~P~33775-N, with additional funding from the \textit{netidee SCIENCE Stiftung}, 2020--2024.}\\ ISTA \and Sammy Jerome\\ Google Research \and Vahab Mirrokni \\ Google Research\and David Saulpic\footnote{Work done while at IST Austria (ISTA). D. Sauplic has received funding from the European Union’s Horizon 2020 research and innovation
programme under the Marie Skłodowska-Curie grant agreement No 101034413.} \\ CNRS, IRIF \and David Woodruff\footnote{Work done while at Google Research in NYC.} \\ Carnegie Mellon University\and Michael Wunder\\ Google Research}
 \title{Data-Efficient Learning via Clustering-Based Sensitivity Sampling: Foundation Models and Beyond}
\begin{document}
\maketitle

\begin{abstract}

We study the data selection problem, whose aim is to
select a small representative subset of data that can 
be used to efficiently train a 
machine learning model. We present a new data selection
approach based on $k$-means clustering and sensitivity sampling. Assuming access to an embedding 
representation of the data with respect to which 
the model loss is H\"older continuous,
our approach provably allows selecting a set of 
``typical'' $k + 1/\eps^2$ elements whose average loss 
corresponds to the average loss of the 
whole dataset, up to a multiplicative $(1\pm\epsilon)$ factor and an additive $\eps \lambda \Phi_k$, where $\Phi_k$ represents the $k$-means cost for the input 
embeddings and $\lambda$ is the H\"older constant. 

We furthermore demonstrate the performance and scalability of our approach on fine-tuning foundation models and show that it outperforms state-of-the-art methods. We also show how it can be applied on 
linear regression, leading to a new sampling strategy 
that surprisingly matches the performances of leverage 
score sampling, while being conceptually simpler and more
scalable.
\end{abstract}

\section{Introduction}
The growth of both datasets and models to a massive
scale has led to 
a new generation of machine learning models with astonishing performance. Yet, 
the size of these models and datasets
makes their training and fine-tuning extremely difficult, costly, time-consuming, and so nearly impossible for most academic institutions
or small-scale companies to perform. On the other hand, 
a complete dataset is often not needed to reach nearly optimal performance (i.e., up to a small increase in error percentage). A central question then becomes how to identify the most important 
data items for the training or fine-tuning process. 

While uniform sampling often shows surprisingly good
performance, it is still suboptimal, especially when
dealing with real datasets that are complex and imbalanced.
To better capture the usefulness of the underlying 
data to train the model, data selection and active learning 
methods deduce which data items are the most relevant 
for training or fine-tuning, based on their uniqueness, 
quality, and the model's knowledge.
There exist several heuristics or greedy approaches for 
active learning and data selection~(see e.g. \cite{dasgupta2004analysis} or references in \cite{ren2021survey}).
State-of-the-art data selection strategies are uncertainty-based, e.g., 
\emph{margin} or \emph{entropy} scores, and aim at selecting items for which the uncertainty of the model is high. However, such purely model-based methods have the
additional overhead of requiring evaluating the model on 
the whole dataset.

In a celebrated result,  Sener and Savarese~\cite{SenerS18}
 showed that state-of-the-art active learning strategies are difficult to use
in modern training frameworks for the following reasons:
\begin{enumerate}
\item The training proceeds in batches, which requires
the active learning strategy to pick \emph{not only one} 
training element at a time but a \emph{batch} of training elements. However, to make the most out of the batch, it is needed to ensure some diversity 
in the set of elements sampled, which often unfortunately anti-correlates with, for example, the margin objective  which may lead to picking near-duplicates elements
(see the discussion in \cite{SenerS18} for more details).
\item The score (i.e., value) of the training elements 
are obtained through the model. This requires running the 
model on the data items to determine which ones to pick 
next. Unfortunately, modern models are often very large and 
the inference time is particularly costly. Moreover, margin scores are not 
well suited for foundation models. 
\end{enumerate}

The solution proposed by~\cite{SenerS18} consists of using the notion of \emph{coresets} to select the data. A coreset is a subset of the data defined
such that optimizing the model on the coreset yields a 
good model for the entire dataset (i.e., good generalization
bounds for the whole dataset). In more formal terms, the average model's gradient (or model loss) of the coreset elements is the same as the model's gradient of the 
whole dataset and, thus, learning on the coreset elements has the same effect as learning on the whole dataset. Unfortunately, to implement this approach one would need to obtain the gradient or loss of \emph{all} the data items, which implies running the model on all the data items. To circumvent this problem, Sener and Savarese show that, given some embedding representation of the dataset and some set of assumptions relating the embeddings to the model's loss, some coreset can be computed using a heuristic to the $k$-center 
objective. Their embedding assumption is a fairly natural one
since the embeddings can be obtained from a pre-trained model
or from
a generic embedding model (e.g., BERT~\cite{devlin2018bert}, word2vec~\cite{mikolov2013efficient}).
This approach has, however, the following suboptimal behaviors:

%\begin{enumerate}
%    \item 
    (1) The first practical issue is that the $k$-center objective is particularly sensitive to outliers, and in particular the greedy 2-approximation algorithm in the work of \cite{SenerS18}. Indeed the algorithm iteratively picks
the training items that are the furthest away (in the embedding space) from the already selected training items. This tends to select outliers, increasing the diversity at the expense of the relevance of the elements. 
We ask: Can we find a more robust way of selecting
a set of items which is both diverse and that precisely covers
the most important traits of the dataset?
%\item 

(2) A second theoretical 
drawback is that the bounds proven are quite weak and require
strong assumptions on the relationship between the embeddings of the training elements and the model loss, in particular on the 
spread of the data elements (see Section~\ref{sec:ourmodel} for
more details). 
We ask: Can we provide a theoretical solution that would require
a minimal set of assumptions on our dataset and model?

%\item 
(3) Finally, and maybe most importantly, their approach is limited to classification tasks. We ask: Can we provide a more generic data-selection algorithm, working for a more general loss function and in particular for
the new generation of foundation models?
%\end{enumerate}

\subsection{Our Approach and Contribution}
Consider the problem of fine-tuning a 
Large Language Model (LLM) on a specialized task,
such as translation. Even though we have abundant
data points for the translation task, it is often 
time-consuming and costly to fine-tune on the whole 
translation dataset, and we would instead prefer to
sample a small representative subset of data that
can still be used to build a high quality model.
While there are methods that compute importance scores
for each data point (e.g. margin scores) that can then 
be used to select data, these scores are expensive to 
compute, since they require evaluating \emph{all} the 
data using the LLM.

Our key insight is to leverage such expensive, accurate scores on a \emph{sublinear} number 
of data points, coupled with less accurate but fast to compute 
embeddings, which can be generated by a much simpler
and efficient model. Surprisingly, we find that
even simple embeddings, such as those generated by a
pre-trained BERT model \cite{devlin2018bert,mikolov2013efficient}, can be
predictive of the loss affinity between different data 
points, for a much larger and more complex models (see 
also Figure~\ref{fig:loss_histogram}).
Our data selection algorithm utilizes clustering and sketching techniques, offering strong theoretical guarantees and significant practical improvements over existing methods on benchmark datasets.
The power of our theoretical contribution is that 
while the analysis is 
quite simple, it significantly improves over previous 
work, in particular over \citet{SenerS18}.

Ideally, we want to sample data proportional to the model loss. However, like margin or entropy scores, this is expensive due to requiring model evaluation on the entire dataset. Instead, we leverage embeddings to identify a diverse and relevant subset.
Our approach begins with $k$-means clustering on the entire dataset. Then, elements are sampled using \emph{sensitivity sampling} on \emph{proxy losses}.
Specifically, the algorithm first computes a $k$-means clustering on the whole dataset and then
samples each element with 
probability proportional to its distance to the closest mean
plus the mean's loss (see \cite{FL11} for the introduction of that probability distribution for clustering coreset, see also~\cite{bachem2018scalable}).  
% for each cluster we select the 
% point that minimize the sum of the distances to the other cluster 
% points. 

Here, we use a $(k, z)$-clustering objective (e.g., $k$-median for $z=1$ and 
$k$-means for $z=2$) because it provides a more robust clustering measure than 
$k$-center as it is much less sensitive to outliers.

In the experiments section~\ref{sec:experimental_results_nn}, we show that 
 this yields a better sample
in practice than what was previously known, in particular for fine-tuning
a foundation model (specifically for fine-tuning an LLM to a translation task).
We further show that our method is quite general as it works for both neural networks and for regression tasks.

Next, we provide theoretical guarantees on this sampling strategy. First, 
assuming there is a clustering such that, for each cluster, the model loss
is H\"older with respect to the embeddings -- a 
well-motivated assumption as we show in the experiments section and less restrictive than the Lipschitz assumption of \citet{SenerS18} -- we can prove that the samples provide 
a strong proxy for the loss of all the data elements. 
More specifically, we show that by using sensitivity sampling
% , meaning sampling each element with 
% probability proportional to its distance to the closest mean
% plus the mean's loss, 
we obtain a coreset whose average model loss is within a $(1\pm\eps)$ factor of the 
average model loss on the whole dataset, plus an additive term 
corresponding
to the loss of the $(k, z)$-clustering objective. This implies that 
if the embeddings of the data is clusterable, we obtain an actual coreset for the model loss with
only few {\em inferences}, i.e., queries to the loss function $\ell$.
Moreover, for classification tasks, expecting that the 
model embeddings will be clusterable is not an unrealistic 
assumption: we do expect that points from the same 
class have closer model-embedding distance than points in different classes. This intuition is validated in
Table~\ref{table:lambda}.

More formally, we work with the notion of Hölder continuity:
%, which generalizes Lipschitzness as follows: 
we say that the loss function $\ell$ is $(z, \lambda)$-Hölder continuous if for any $x, y$ with embedding $E(x)$ and $E(y)$, $|\ell(x) - \ell(y)|\leq \lambda \|E(x)-E(y)\|^z$.
We make the following assumption on the loss function:
%Formally, under the assumption that the loss function $\ell$ is $(z, \lambda)$-Hölder continuous, namely for all $x, y$, $|\ell(x) - \ell(y)|\leq \lambda \|x-y\|^z$, we get: 
 
\begin{assumption}\label{ass:lip}
For $\Lambda = (\Lambda_1, ..., \Lambda_k) \in \R^k$, an embedding $E$ and a $k$-clustering of the input $\calC$, we say the loss function is $(z, \Lambda)$-well-behaved with respect to  $E$ and  $\calC$ when, 
for any cluster $C_i$ and point $e \in C_i \cap S$, $|\ell(e) - \ell(c_i)| \leq \Lambda_i \|e - c_i\|^z.$
%restricted to cluster $i$, the loss function is $(z,\Lambda_i)$-Hölder continuous. 
\end{assumption}
Note that this definition generalizes Lipschitzness: if the loss function is $\lambda$-Lipschitz, then the above holds for $\Lambda_i = \lambda$ and $z=1$, regardless of the clustering $\calC$.

In Table~\ref{table:lambda}, we validate Assumption~\ref{ass:lip} on standard 
benchmark datasets. We cluster the data, and then
%for each cluster compute the percentiles of the ratio
 compute the percentiles of the ratio
$\left|\ell(e) - \ell(c_i)\right| / \left\|e-c_i\right\|^z$ (for MNIST $\ell$ is the loss function,
while for GAS it is the target variable).
We observe that these values are bounded,
implying that the MNIST and GAS datasets do possess the
$(z,\lambda)$-H\"older condition.
\begin{table}[h]
    \caption{Value of the H\"older continuous constant  for $z = 2$ the different
    percentiles of the MNIST (classification) and GAS (regression) datasets.}
    \label{table:lambda}
\centering
      \begin{tabular}{l|c|r}
      Smallest \%ile & $\lambda$ for MNIST & $\lambda$ for GAS \\ \hline
20 & 0.00111 & 0.02286\\ \hline
40 & 0.00320 & 0.04488 \\ \hline
60 & 0.00663 & 0.07158\\ \hline
80 & 0.01416 & 0.12589\\ \hline
99 & 0.05935 &  0.86976\\ \hline
      \end{tabular}
\end{table}

We also sanity check our assumption on LLMs. 
Specifically, we examine how predictive a BERT-based
embedding clustering is of the losses on a much larger
T5 transformer model. Figure~\ref{fig:loss_histogram}
shows that on average a data point's loss on the T5 
model is much closer to the loss of similarly clustered
points, than that of random points.
\begin{figure}[h]
\includegraphics[width=0.9\linewidth]{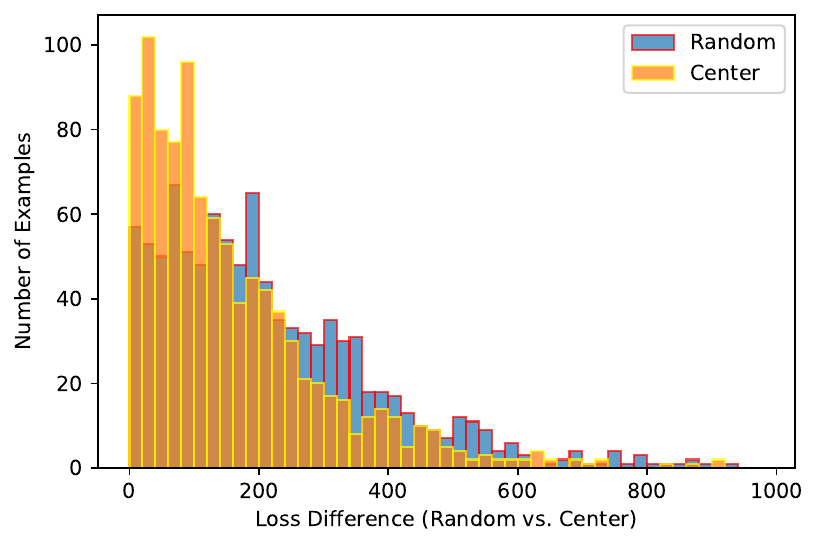}
\caption{Distribution of loss to random point vs center of corresponding cluster for the WMT T2T EnDe translation dataset  ~\cite{bojar-EtAl:2014:W14-33} using BERT embeddings ~\cite{devlin2018bert}.}
\label{fig:loss_histogram}
\end{figure}

We now proceed to state our main theorem, which gives
strong data selection guarantees under Assumption~\ref{ass:lip}.
\begin{thm}\label{thm:main}%[See formal statement in \Cref{thm:mainPrecise}]
  Let $\eps,z > 0$, $\Lambda \in \R^k$. Let $\calD$ be a dataset and $\ell$ a loss function that is $(z, \Lambda)$-well-behaved with respect to an embedding $E$ and a clustering $(C_1, ..., C_k)$ into $k$ clusters. 
  Then, there exists an algorithm that makes 
  $k$ queries to $\ell$ and outputs a 
  sample $S$ of size $O(\eps^{-2})$
  and a weight function $w$ such that
  \begin{align*}
      &|\sum_{e \in \calD} \ell(e) - \sum_{e \in S} w(e) \ell(e)| 
      \le 
  \eps\left(\sum_{e \in \calD} \ell(e) + 2 \cost^\Lambda_{\calC, z}(\calD)\right),
  \end{align*}
  with constant probability, where $\cost^\Lambda_{\calC, z}(\calD) = \sum_{i=1}^k \Lambda_i \cost_{1,z}(C_i)$, namely the $(k,z)$-clustering cost where cluster $i$ is weighted by $\Lambda_i$.
\end{thm}

%On the other hand, $\cost_{k, 2}$ is the $k$-means cost of $\calD$. 
To interpret this result better note that when the loss function is $(\lambda, z)$-Hölder continuous, the upper-bound becomes $\eps ( \sum_{e \in \calD} \ell(e) + 2 \lambda \cost_{k, 1}(\calD))$, where $\cost_{k, z}(\calD)$ is the $(k,z)$-clustering cost of the dataset (which is $k$-means when $z=2$, or$k$-median when $z=1$ and the loss is $\lambda$-Lipshitz). 
We believe that our condition is quite a weak requirement; furthermore, we 
show that such a condition is needed: {\em There exist worst-case loss 
functions that are not Hölder continuous and for which it is necessary to 
query the whole dataset in order to get the above theorem statement}
(see \cref{thm:nonAdaptive}). Thus, this answers the second question raised 
above.  

The assumption on the clustering can be read as follows: we expect the embedding to have a clustered structure that an algorithm (e.g., $k$-means++) can discover. {\em We expect elements in each cluster to be similar, and therefore that the loss function is smooth within each cluster} -- this is formalized by the Hölder continuity. Note that the clustering may not be optimal, we only need to be able to compute it efficiently.

\begin{comment}
Except for this assumption, the loss function and the structure of the problem can be quite general.
%: therefore, this answers question 4. 
We complement this with a similar theorem for linear regression. Combined it completes our answer for question 4. 
\end{comment}

The additive error in our results depends on the $(k, z)$-clustering cost, which is defined as $\cost_{k, z}(\calD) := \min_{|C| = k} \sum_{e \in \calD} \min_{c \in C} \|e-c\|^z$. 
Therefore, depending on the choice of $z$, our upper-bound can be made more, or less, robust to outliers: the smaller the $z$ the more resilient to outliers it becomes; for sufficiently large $z$ the objective becomes the $k$-center objective of \cite{SenerS18}, which would
translate here (with the assumption that the loss is $\lambda$-Lipschitz) into an upper bound $n \cdot \lambda \cdot \min_{|C| = k} \max_{e \in \calD} \min_{c \in C} \|e-c\|$. This addresses question 1 raised above.

We further demonstrate that the resulting sampling strategy outperforms classic data selection approaches, namely, training the model using the set $S$ obtained via \Cref{thm:main} gives a several percentages increase in  accuracy (more than 4\% for Fashion MNIST) than using other methods. Similarly, for linear regression, we show empirically that our sampling strategy is competitive and sometimes outperforms more sophisticated state-of-the-art methods, such as leverage score sampling, adding a fundamentally new sampling strategy to the growing body of work on active regression 
\cite{CP19,CD21,PPP21,MMW022,WY23}.

We defer a detailed survey of related work to~\Cref{appendix:further_related_work}

\section{Problem Formulation}
Given a dataset $\calD$  and a 
machine learning model, the high-level goal is to find a 
subset $S$ of $\calD$ such that training the model on $S$ yields
approximately the same model as training the model on $\calD$, while the time taken to compute $S$ and train the model on $S$ 
should be much smaller than the time taken to train the model 
on $\calD$.

We focus here on the general data selection problem, and dedicate \Cref{sec:regression} to the special case of linear regression.

\subsection{Our Model}
\label{sec:ourmodel}
We assume that we are given a dataset $\calD$ of size $n$, together with a loss function 
$\ell$ such that $\ell(e)$ is the loss of the model on instance $e$.
%gradient of the loss function
%of the model on instance $e$. \todo{DS: need to explain why we work with the gradient (c.f. remark below)}
%\todo{VCA: or simply the loss function? Or what makes the most sense here?} 
%We can make the high-level problem statement more formal as follows.
The goal is to sample $S \subseteq \calD$ of limited  size, 
and associate a weight function $w : S \mapsto \mathbb{R}_+$ such that
\[\Delta(S) := |\sum_{e \in \calD} \ell(e) - \sum_{e \in S} w(e) \ell(e)| \le \delta,\]
for the smallest possible $\delta$. Note that $\ell(e)$ can be queried 
simply by running the model on $e$ and computing the loss 
for $e$, which is expensive. This is why we want to compute $S$ without having to compute $\ell(e)$ for 
all $e \in \calD$.
% \todo{VCA: I think this is pretty generic formulation but I would love to hear more expert thoughts.}
%
We now provide a complete formulation of the problem. 
\begin{definition}[Data Selection, \cite{SenerS18}]\label{def:dataSelect}
The data selection problem is defined as follows:
\begin{itemize}
    \item 
  \textbf{Input:} A dataset $\calD$, an oracle access to a function $\ell : \calD \mapsto \mathbb{R}_+$, and a target size $s$.
  \item \textbf{Output:} A sample $S \subseteq \calD$ of size at most $s$
  together with a weight function $w:S \mapsto \mathbb{R}_+$ such that
  \begin{itemize}
      \item The number of queries to $\ell$ (i.e.: inferences) is at most $s$.
      \item $S$ minimizes
      \begin{equation}\label{eq:est}
          \Delta(S) := |\sum_{e \in \calD} \ell(e) - \sum_{e \in S} w(e) \ell(e)|.
        \end{equation}
  \end{itemize}
  \end{itemize}
\end{definition}
Note two differences from the the original definition of \cite{SenerS18}. 
(A) First, they use uniform weights, namely $\forall s \in S, w(s) = \frac{|\calD|}{|S|}$ (in which case, minimizing \Cref{eq:est} means that the average $\ell(e)$ in the sample should be close to the average $\ell(e)$ for the whole data). We slightly generalize the definition to allow for different sampling strategies, while keeping an unbiased estimator.

(B) Second,  \cite{SenerS18} consider the loss after re-training the model with $S$, namely $\left| \frac{1}{n} \sum_{e \in \calD} \ell(e, \calA(S)) - \sum_{e \in S}w(e) \ell(e, \calA(S))\right|$. In words, the loss of the model trained on $S$ is roughly the same evaluated on $S$ as on $\calD$. In order to bound this quantity, Sener and Savarese make {\em strong assumptions on the distribution of the dataset}, namely the labels are drawn randomly from a structured distribution, and it is further assumed that the training loss is $0$ on their sample.
Instead, we stay more general and focus on the loss in the current model. Our underlying assumption is that, if $S$ approximates the loss well, then it contains ``typical'' items of the dataset (with respect to the current model), and therefore training the model based on $S$ should be similar as  training it on $\calD$. 
This formulation allows us to show strong theoretical results for the data selection problem, {\em without any assumptions on $\calD$}. 
%This contrasts with \cite{SenerS18}, who needs to assume $\calD$ is drawn randomly from a constrained distribution.
Note that our objective is more challenging than the one from Sener and Savarese: The bound we prove on $\Delta(S)$ implies
the result of~\cite{SenerS18} (under their assumption about
the model loss).

% \cite{SenerS18} formalized the problem in this way as the term $\Delta(S)$ naturally appears as an upper-bound on the active learning loss, namely as the expected future loss of the model trained on the sample. More precisely, when the dataset consists of random samples from an unknown distribution $\rho$, their task is to minimize the expected future loss $\E_{x \sim \rho}[\ell(x, \calA(S))]$, where $\calA$ is a learning algorithm that outputs a set of parameters for the model, $S$ is the sample and $\ell(\cdot, \calA(S))$ is the loss function given those parameters. They note that this expectation can be upper-bounded with
% \begin{align}
% \notag
%     \E_{x \sim \rho}[\ell(x, \calA(S))] 
%     \leq \underbrace{\left|\E_{x \sim \rho}[\ell(x, \calA(S))] - \frac{1}{n} \sum_{e \in \calD} \ell(e, \calA(S))\right|}_{\text{Generalization error}}
%     + \underbrace{\left|\sum_{e \in S}w(e) \ell(e, \calA(S))\right|}_{\text{Training Error}}\\
%     \label{eq:futureLoss}
%     + \underbrace{\left| \frac{1}{n} \sum_{e \in \calD} \ell(e, \calA(S)) - \sum_{e \in S}w(e) \ell(e, \calA(S))\right|}_{\text{Coreset Error}}
% \end{align}
% Since Generalization and Training error are well understood, \cite{SenerS18} state that the critical part of active learning is to handle the Coreset error, which motivates \Cref{def:dataSelect}.

\subsection{Assumptions on $\ell$}
\paragraph{Limits to the general case}
The above formulation requires that the number of queries to $\ell$ is sublinear in $|\calD|$.
Unfortunately, as long as $s = o(|\calD|)$ it is impossible to bound $\Delta(S)$ without further assumptions on $\ell$ or $\calD$, which can be seen by the following worst-case instance: The adversary chooses
    uniformly at random (u.a.r.) an element $e^* \in \calD$ and defines $\ell(e^*) = 1$ and 
    $\ell(e) = 0$ for all $e \neq e^*$. Then computing with constant success probability a sample $S$ 
    of size $s = o(\calD)$ such that $\Delta(S) = o(\sum_{e \in \calD} \ell(e))$ 
    with $o(|\calD|)$ queries to $\ell$  is 
    impossible.
Of course, this worst-case instance is unrealistic and we can hope to better capture 
the structure of real-world dataset.
    
\paragraph{Assumption on the embeddings}
In practice we can assume that each element $e$ in $\calD$ 
could be associated with a vector $v_e$ in $\mathbb{R}^d$ for some 
$d$, possibly coming from a model that is ``well-behaved'' with 
respect to the loss function of the model.
More concretely, the embeddings of the data elements can either be obtained from a generic embedding of the input dataset $\calD$, 
e.g., the BERT or word2vec embeddings for words~\cite{devlin2018bert,mikolov2013efficient}, 
or an embedding obtained through the last layers of the
model being trained. The last assumption is particularly
realistic in the \emph{warm start} or \emph{fine-tuning} regime where the model has 
already be partially trained.

\paragraph{H\"older Continuity assumption} This is the assumption presented above as Assumption~\ref{ass:lip}.
Lipschitzness is a common assumption, and is theoretically grounded for some embeddings (see e.g. Lemma 1 in \cite{SenerS18} for CNN).
Our assumption relaxes Lipschitzness, and our theoretical finding are therefore more general.

In the following, to ease notation for each element $e \in \calD$, we will
also use $e$ to denote its embedding in $\R^d$.

The problem we consider throughout the rest of the paper is the \emph{Data Selection under well-behaved loss} problem, which is the problem of \Cref{def:dataSelect} when the loss function $\ell$ is well behaved (see formal definition in \Cref{ass:lip}).
 This definition can be extended to the active learning setting
where the objective is to iteratively choose a set of elements to
sample based on the model updates.
\begin{definition}[$r$-Adaptive Active learning under well-behaved norm]
The $r$-adaptive active learning problem under $(z, \lambda)$-Hölder Continuity is 
defined as follows:

\begin{itemize}
    \item 
      \textbf{Input:} A set of elements $\calD \subset \R^d$, an oracle access to a function
      $\ell : \calD \mapsto \mathbb{R}_+$ that is well-behaved, a target size $s$ and an adaptivity parameter $r$.
  \item \textbf{Adaptivity:} There are $r$ rounds. At round $i$, 
  the algorithm can query $\ell$ on a set $Q_i$ of size at most $s$. $Q_i$ can 
  only be defined based on the results of $\ell$ on $\cup_{j < i} Q_{j}$ and $\calD$.
  \item \textbf{Output:} For all $i \in [r]$, a sample $S_i \subseteq \calD$ of size at most $s$
  together with a weight function $w_i:S \mapsto \mathbb{R}_+$ such that
 % \begin{itemize}
     % \item For each $i$ the  number of queries to $\ell$ is at most $s$.
     % \item 
      $S_i$ minimizes \[\Delta(S) := |\sum_{e \in \calD} \ell(e) - \sum_{e \in S_i} w_i(e) \ell(e)|.\] 
  %\end{itemize}
  \end{itemize}
\end{definition}

\subsection{Clustering Preliminaries}
We defer a detailed description on clustering preliminaries to \Cref{appendix:clustering_preliminaries}. Most importantly,  the $(k, z)$-clustering cost of $C$ on $\calD $ is $\cost_z(\calD , C) := \sum_{x \in \calD } \min_{c \in C} \|x-c\|^z$
and $\cost_{k, z}(\calD ) := \min_{C \subset \R^d,~|C| \le k} \cost_z(\calD ,C)$. 
For $z=1$, this objective corresponds to $k$-median, while for $z=2$ it corresponds to $k$-means. We call a clustering $\calC$ any partition of $\calD$ into $k$ parts (called clusters) $C_1, ..., C_k$. Given a $\Lambda \in \R^k$, we define $\cost^\Lambda_{\calC, z}(\calD) = \sum_{i = 1}^k \Lambda_i \cost_{1,z} (C_i)$.

\section{Algorithmic Results}
We now study sampling procedures for the active learning problem, defined in the previous section.
Our goal is to build a sampling strategy such that $\sum_{s \in S} w(s) \ell(s)$ is an unbiased estimator of $\sum_{e \in \calD} \ell(e)$, and show that the estimator is tightly concentrated around its mean. We will first study the case where it is not possible to query the loss function at all, and show a lower bound on the error achievable. We then present some adaptive algorithms, which query  the loss function sparingly.

% $\var(X) := \sum_{x \in X} ||x-\mu(X)||^2$. Given a set of $k$ point $C \subset \R^d$,
% we denote the $k$-means cost of $C$ on $X$ as $\cost(X, C) := \sum_{x \in X} \min_{c \in C} ||x-c||^2$
% and $\cost(X) := \min_{C \subset \R^d,~|C| \le k} \cost(X,C)$. Note that $\var(X) = \cost(X)$ for $k=1$.
% Without loss of generality, we assume that the distance between any pair of points
% in $X$ is in $[2, \Delta]$ for some large value $\Delta$. 

\subsection{Algorithm and Lower Bound for the Non-Adaptive Case}
We first focus on the context where the algorithm cannot query function $\ell$ at all. In this case, if one only assumes the loss function to be Hölder continuous the error must scale linearly with both the size of the dataset and its diameter, and this can be achieved by a random sample of the data points, as we show in the following theorem. The proof is deferred to \Cref{appendix:proof_non_adaptive}

\begin{thm}\label{thm:nonAdaptive}
  Let $\eps,\lambda > 0$. 
  There is a constant $c$, a dataset $\calD$ and a loss function $\ell$ that is $(1, z)$-Hölder such that, when $S$ is a uniform sample of size $1/\eps^2$ with weight function $w(e) = n/s$, it holds with constant probability that
  \[\Delta(S)  = |\sum_{e \in \calD} \ell(e) - \sum_{e \in S} w(e) \ell(e)|\ge c \eps n \sup_{e \in \calD} \ell(e).\]
  
  Furthermore, this lower bound is tight: for all dataset $\calD$ and loss function $\ell$,
  a uniform sample $S$ of size $s = O(1/\eps^2)$ with weights $w(e) = n/s$ satisfies with constant probability
  $\Delta(S) \le \eps n \sup_{e \in \calD} \ell(e).$
 \end{thm}

\subsection{Adaptive Algorithms}
The lower bound on  $\Delta(S)$ in \Cref{thm:nonAdaptive} shows  is that one must sample more carefully if good guarantees are desired. As explained in the introduction, we present a sampling strategy that queries at most $O(k)$ many points, and reduce the additive error to $\eps \lambda \cost_k(\calD)$. This is always better than $\eps n \, \underset{e \in \calD}{\sup}\, |\ell(e)|$, and, in case the embedding of $\calD$ has a clustered structured, can be drastically smaller.

\subsubsection{1-Round Algorithm}
In this section
we state the $1$-round algorithm and show \Cref{thm:main} (with a full proof deferred to \Cref{appendix:proof_thm_mainPrecise}). We start by the pseudo-code of the algorithm.

\begin{algorithm}
\caption{Data-Selection($\calD, k, \eps, \Lambda, \calC$)}
\label{alg:oneRound}
\begin{algorithmic}[1]
\STATE \textbf{Input}: a dataset $\calD$ partitioned into clusters $\calC = (C_1, ..., C_k)$ with centers $c_1, ..., c_k$ and a $k$-tuple of parameters $\Lambda_1, ..., \Lambda_k$.
\STATE For $e \in C_i$, define %$\calA(e) = \argmin_{a \in \calA} \|e-a\|$ the element of $\calA$ 
   % that is the closest to $e$,  
   $\hat{\ell}(e) := \ell(c_i)$ and 
    $v(e) := \|e-c_i\|^z$.
\STATE Let $s:= \lceil\eps^{-2} (2+2\eps/3)\rceil$. For $e \in C_i$ define $p_e := \frac{\hat{\ell}(e) + \Lambda_i v(e)}{\sum_i \Lambda_j \cost(C_i, \{c_i\}) + \sum_{x \in \calD} \hat{\ell}(x)}$ and $w(e) = s^{-1} p_e^{-1}$.
\STATE Compute a sample $S$ of $s$ points, picked independently following the distribution $p_e$. 
\STATE \textbf{Output:} the set $S$ with weights $w$.
\end{algorithmic}
\end{algorithm}

The proof of \cref{thm:main} works as follows: we let $X_i$ be the random variable corresponding to the contribution of the $i$-th sample to the cost. $\sum X_i$ is therefore an unbiased estimator for $\sum_{e \in \calD} \ell(e)$: we show that each $X_i$ has a small variance, and apply Bernstein's inequality to conclude. We formalize this argument in \Cref{appendix:proof_thm_mainPrecise}.

\begin{rem}
    Instead of requiring that, in each cluster, the worst-case $\frac{|\ell(x)-\ell(c_i)|}{\|x-c_i\|^z}$ is bounded, we can allow for a few outliers and require only that $\frac{|\ell(x)-\ell(c_i)|}{\|x-c_i\|^z} \leq \Lambda_i$ for all $x$ but a $1/k$-fraction of the probability mass defined line 3 of \cref{alg:oneRound}. This allows to have some outliers, 
\end{rem}

\subsubsection{$r$-Round Algorithm}
We now turn to obtain better guarantees than the above bounds by
allowing for more rounds. Here, we assume that we are given a set of centers $c_1, c_2, ...$ such that for all $k$, the set $(c_1, ..., c_k)$ is a good solution to $(k,z)$-clustering. Note that this is precisely the guarantee of the $k$-means++ algorithm (and, more generally, $D^z$ sampling). We let $\calC_k$ be the set of clusters corresponding to the centers $c_1, ..., c_k$.

\begin{thm}\label{thm:rounds}
  Let $\eps > 0, \Lambda \in \R^k$, and integer $r > 0$. Let $\calD$ be a dataset with a set of centers $c_1, ..., c_{kr}$, and $\ell$ be be a loss function that is well-behaved with respect to $\Lambda$ and $\calC_i$, for all $i \in \{k, 2k,...,kr\}$.
  Then
  there exists an algorithm that  for each round $i \in [r]$, queries $k$ elements per round and outputs a sample $S_i$ of size at most $O(1/\eps^2)$
  and a weight function $w_i$ such that:
  \begin{align*}
      \Delta(S_i) &= |\sum_{e \in \calD} \ell(e) - \sum_{s \in S_i} w_i(s) \ell(s)| \\
      &\le \eps\left(\sum_{e \in \calD} \ell(e) + \cost^\Lambda_{\calC_{ik}, z}(\calD)\right)
  \end{align*}
\end{thm}

Note that the above algorithm allows to trade-off the round complexity and 
sample size and reaches optimality in the limit: when $r \cdot k = |\calD|$, we 
obtain an exact algorithm. 
The algorithm is very similar to \Cref{alg:oneRound}: We defer the presentation to \Cref{app:rounds}.

%\begin{enumerate}
%% \item Compute the element $\hat{x}$ that is the closest to the mean $\mu(X)$ of $X$.
%% \item Query the oracle on $\hat{x}$ to obtain $\ell(\hat{x})$.
%\item Compute an $O(1)$-approximation to prefix $k$-means clustering; i.e.: an
%ordering of the points in $\calD$ such that any prefix of length $k_0$ is an
%$O(1)$-approximation to $k_0$-means on $\calD$.
%\item At round $i$, query $\ell$ for the points at position in $[(i-1)k + 1, ik]$
%in ordering.
%Define $\calA$ to be the set of elements at position at most $ik$ in the 
%ordering.
%For each element $x\in \calD$, let $\calA(x)$ be the element of $\calA$ 
%    that is the closest to $x$, let $\hat{\ell}(x) := \ell(\calA(x))$ and 
%    $v(x) := ||x-\calA(x)||^2$.
%    \item Pick a sample $S \subseteq X$ where each point $x \in X$ is sampled with probability $\min(1, p_x)$, where 
%    \[p_x := \eps^{-2} \frac{\hat{\ell}(x) + \lambda v(x)}{\sum_{x \in X} \hat{\ell}(x) + \lambda \cost(X)}\]
%    and define $w_i(x) = \max(1, p_x^{-1})$ for all $x$ sampled.
%\item Output $S_i, w_i$ for round $i$.
%\end{enumerate}

\subsubsection{Computing the clustering $\calC$ and the parameter $\Lambda$}

Our algorithms require the knowledge of a clustering $\calC$ and the vector of parameters $\Lambda$. We explain here how to compute those values.

\paragraph{Finding a Clustering}

Our theorems require that the loss is well-behaved w.r.t an estimate of $\Lambda$ and a clustering $(C_1, ..., C_k)$. 
To compute such a clustering, we can use any algorithm for $(k,z)$-clustering, e.g., $D^z$-sampling (the generalization of $k$-means++ \cite{ArthurV07}), or some faster algorithms, e.g. \cite{DBLP:conf/nips/Cohen-AddadLNSS20,DBLP:conf/nips/Cohen-AddadLNSS21}.

\paragraph{Estimating $\Lambda$.}
Once we are given a clustering, we can query the loss function in order to estimate $\Lambda$. Formally, we have the following:
\begin{lem}\label{lem:upperBoundLambda}
    Assume that there is a probability $p$ such that, for each cluster $C_i$, 
    \begin{align*}
        \Pr_{x \in C_i}\left[\frac{|\ell(x)-\ell(c_i)|}{\|x-c_i\|^z} \in [\Lambda_i/\log(n), \Lambda_i]\right] \geq p,
    \end{align*}
    where the probability is taken over an $x$ chosen uniformly at random from $C_i$. Then, one can compute an upper bound on each $\Lambda_i$  with probability $99/100$  by querying the loss of $\log(100 k) / \log(1-p)$ points per cluster.
\end{lem}
The proof is deferred to Appendix~\ref{app:lambda}.
We note that we could have made different assumptions in the previous lemma: the $\log(n)$ is picked somewhat arbitrarily, to fit with an exponentially decreasing tail on the distribution of the ratios $\frac{|\ell(x)-\ell(c_i)|}{\|x-c_i\|^z}$. We believe that this assumption is quite natural, and indeed our experiments confirm it across different applications 
(see \cref{table:lambda}).

\section{Data Selection for Regression}
\label{sec:regression}

In this section, we specialize our method, i.e., sampling according to $(k,z)$-clustering cost, to the setting
of linear regression. 
Ideally, given a matrix $A$, our goal is to compute a sketching and rescaling diagonal matrix $S$ with as few non-zero entries  as possible such that computing the optimal regression on $S$ is equivalent to computing it on $A$. For this, we are seeking a ``coreset guarantee'', namely we want  $\|SAx - b\| \approx \|Ax-b\|$, for all $x$. In the following $a_i$ denotes the $i$-th row of $A$.
Therefore, we define the data selection problem for regression as follows:

\begin{definition}
The data selection problem for linear regression is defined as follows:
\vspace{-1em}
\begin{itemize}
    \item \textbf{Input:} a $n \times d$ matrix $A$ and an $n$-dimensional vector $b$, and a target number of queries $k$. 
    \item \textbf{Output:} A sample $S \subseteq [n]$ of size at most $s$ together with a weight function $w : S {\rightarrow} \R_+$ such that 
    $\left|\sum_{s \in S} w(s) (\langle a_s, x\rangle  - b_s)^2 - \|Ax-b\|_2^2\right|$ is as small as possible, for all $x$.
\end{itemize}
\end{definition}

We show in \cref{sec:experimental_results_regression} that our method is faster and provides results as accurate of the state-of-the-art data selection mechanisms. We provide here some theoretical explanations for this success.

As is the case for the previous active-learning problem, we 
%cannot achieve such a general statement, and 
need to make several assumptions, both on $A,$ and $ b$ and to restrict the set of possible $x$. Our first set of assumptions, similar to \cref{ass:lip}, is the following: 

\begin{assumption}\label{ass:regression}
    For all $i, \|a_i\|_2 = O(1)$ and $b_i = O(1)$. 
    Given $\Lambda \in \R^k$ and a $k$-clustering $\calC = (C_1, ..., C_k)$ of the indices, we say that the input is well-behaved w.r.t $\Lambda$ and $\calC$ when, for every index $j$  in cluster $C_i$ $|b_i - b_j| \leq \Lambda_i \|a_i - a_j\|_2^z$ (where $(a_i, b_i)$ is the center of cluster $C_i$). 
\end{assumption}
The above assumption is similar to the Hölder-continuity  assumption we made for \cref{thm:main}: it formalizes that in the sample, the labels (i.e., $b_i$s) must be close to those of their centers when their embeddings (i.e., the $a_i$s) are close. 
As before, this is necessary to get any result querying a sublinear number of labels $b_i$. This assumption is also related to that of \cite{SenerS18} who assume the labels of the data are drawn randomly, following distributions that are Lipschitz. 

The basic idea of our algorithm is to interpret each row of $A$ as a point in $\mathbb{R}^d$ and cluster these points using $k$-median. Then we compute the optimal regression $x_0$ for the dataset consisting of all centers, each weighted by the size of its cluster, and use $x_0$ to define a probability distribution over all points. Sampling $s$ points according to this distribution gives a set $S$ together with a suitable weight function.

%We write $a_1, ..., a_n$ the lines of $A$, and abusively denote $A = \{a_1, ..., a_n\}$.

\begin{algorithm}
\caption{Data-Selection-Regression($A, k, \eps, \Lambda, \calC$)}
\label{alg:regression}
\begin{algorithmic}[1]
\STATE \textbf{Input:} a matrix $A$ representing the dataset, a clustering $\calC = (C_1,..., C_k)$ of the dataset, and a $k$-tuple of parameters $\Lambda_1,..., \Lambda_k$. 
% \STATE Compute a $O(1)$-approximation $\calA$ to $k$-median on $A$. 
\STATE For all $i \in [n]$, let $j$ be such that $a_j$ is the center of $a_i$'s cluster: define $\hat a_i = a_j$, $\hat{b}_i = b_j$ and  the function $v(a_i, x) = (\langle \hat a_i, x\rangle  - \hat b_i)^2$.
\STATE Compute the optimal regression $x_0$ for the dataset $\{\hat a_1, ..., \hat a_n\}$, i.e., the dataset where each center of $\calA$ is weighted by the size of its cluster.
\STATE For $i$ in clustering $C_j$, define $p_i := \frac{\Lambda_j \|a_i - \hat a_i\|  + v(a_i, x_0)}{\sum_{j' \in [k]} \Lambda_{j'} \cost (C_{j'}) + \sum_{i' \in [n]} + v(a_{i'}, x_0)}$,  and $w(i) = s^{-1} p_i^{-1}$.
\STATE Compute a sample $S$ of $s$ points, picked independently following the distribution $p$.
\STATE \textbf{Output:} the set $S$ with weights $w$.
\end{algorithmic}
\end{algorithm}

Our main theorem for regression is stated next.
%states that, under some assumption on the input data, our data-selection algorithm provides some coreset guarantee. 
The proof is deferred to \cref{app:regression}.

\begin{thm}\label{thm:regression}
    Let $\Lambda$, $A$ and $b$ respect \cref{ass:regression}  for a clustering $\calC$, with $\Lambda_i$ being constants, and let $\hat a_j$ and $x_0$ be as computed by \cref{alg:regression}. 
    
    Let $\calX$ be the set of vectors $x$ such that $\|x\|_2 = O(1)$ and $\forall j \in C_i, |\langle  \hat a_j, x-x_0\rangle | \leq \Lambda_i \|a_j - \hat a_i\|_2$.
    For $s= O(d/\eps^2 \log(1/\delta))$, it holds with probability $1-\delta$ that, for all $x \in \calX$,
    \vspace{-1em}
    \begin{align*}
    &\left|\sum_{s \in S} w(s) (\langle a_s, x\rangle  - b_s)^2 - \|Ax-b\|_2^2\right| \leq \\&\eps (\|Ax-b\|_2^2 + \cost^\Lambda_{\calC, 1}(\calD))
    \end{align*}
\end{thm}

\section{Experiments}
\label{sec:experiments}

We first present results on
neural networks: for a LLM translation task, and then for image classification. 
Finally, we present in Section~\ref{sec:experimental_results_regression}
experiments on a linear regression task.%, and in
% Section~\ref{sec:experimental_results_nn} we present results on
% neural networks: first for a LLM translation task, and then for image classification.

\subsection{Experimental Setup}

For the clustering required in Algorithm~\ref{alg:oneRound},
we run $k''$-means clustering using 
python's sklearn implementation, for some $k'' < k$ on the model's
last layer embeddings. Note that $k$ is the total number of data points to be selected,
while $k''$ is the number of cluster centers sampled (which are a subset of the points 
being selected).
For our experiments, we chose $k''=0.2 k$. Since the
$k''$ cluster centers might not be actual data points from the
dataset, we replace each center with the closest data point 
from the dataset in $\ell_2$ norm (note that by triangle inequalities, this loses only a factor $4$ in the $k$-means cost). After computing $\ell(e)$ for each 
center and extrapolating to the whole dataset using the 
approximation
$\widetilde{\ell}(e) := \ell(c_e) + \lambda \left\|e-c_e\right\|_2^2$ (where $c_e$ is the closest center to $e$),
we sample the remaining $k-k'-k''$ data points proportional to 
$\widetilde{\ell}$.

%While the second goal is a common goal in the data selection 
%literature, many methods violate the first goal, since many of
%them require a \emph{linear} number of model inferences 
%(one per data point). A notable exception is the coreset 
%algorithm of~\cite{SenerS18}, even though it still requires
%significant time to maintain a distance oracle.

\subsection{Experiments on Neural Networks}
\label{sec:experimental_results_nn}

\subsubsection{Fine-Tuning Large Language Models}
We use Algorithm~\ref{alg:oneRound} to select a sample on which to fine-tune an LLM for a translation task. We use the WMT T2T EnDe translation task dataset~\cite{bojar-EtAl:2014:W14-33} which consists of $4,592,289$ training examples, a test set of 
size $3003$ and a validation set of size $3000$. We deduplicated any repeated examples to clean the dataset. We fine-tune a T5-Small model~\cite{t5paper} with 77M parameters (details in \ref{apndx:llm-model-details}). To quantify the effect of the quality of the embedding used, 
we experiment with two different embeddings for the input, 
(1) BERT~\cite{devlin2018bert} and (2) Sentence-T5~\cite{ni2021sentence}. We run three different methods to subsample approximately $1\%$ of the data ($45000$ training examples).
We ran sensitivity sampling
with $k=4500$ and subsample $45000$ elements of the data and a Hölder constant
of 0.1 \footnote{We also experimented with Hölder constants up to 500. See \ref{apndx:llm-holders} for more details.}. 
The diversity sampling methods resembles the one of~\cite{SenerS18}: it consists of running $k$-means (instead of $k$-center), with $k=45000$ and using the elements closest to centers for training. Random is a uniform sample of the dataset of size $45000$. Random-deduped is a uniform sample as well, ensuring no duplicates. We show that our methods drastically improve over uniform or diversity and observe that the results are consistent across the two types of embeddings used (see Figure~\ref{fig:LLMresults}). 

In particular, we note: fine-tuning on the full dataset for $100,000$ steps yields an evaluation accuracy of $0.7$; as such, we improve from a $0.59$ random baseline to $0.6$; covering $9$\% of the headroom.

\begin{figure*}[tb]
\centering
\includegraphics[width=0.45\linewidth]{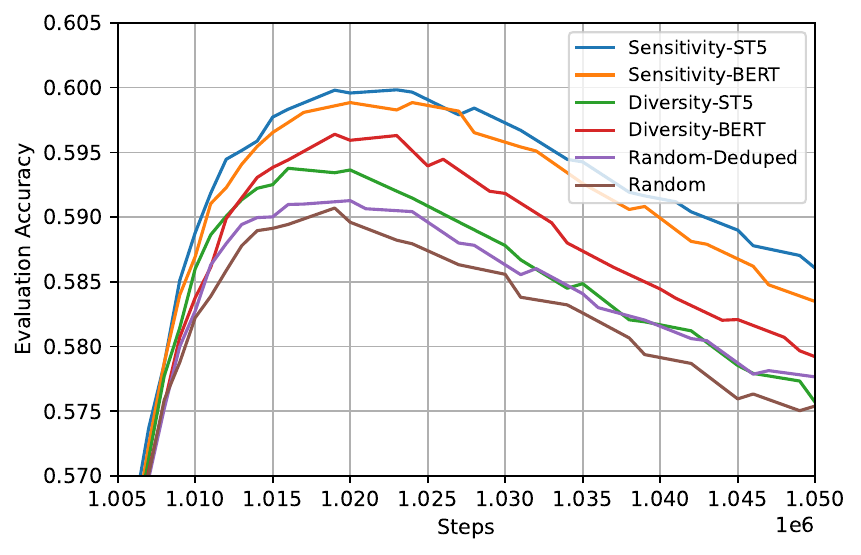}~~
\includegraphics[width=0.45\linewidth]{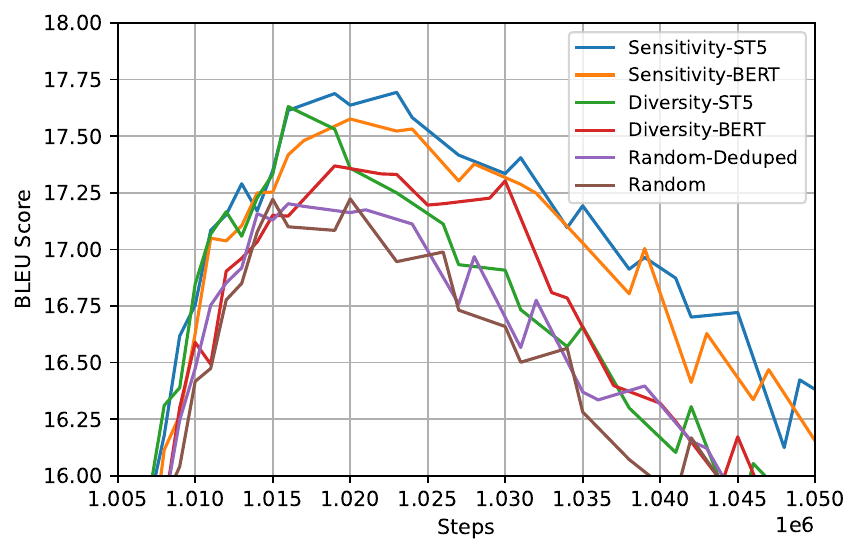}
\caption{Experimental results on the WMT T2T EnDe translation task dataset. We 
report the accuracy (left) and BLEU score (right) of the different methods used: Our method (Sensitivity) compared to Diversity (similar to~\cite{SenerS18}), Uniform cleaned (Random-Deduped), and Uniform (Random). 
Each method is required to produce a sample of roughly $1\%$ of the whole
dataset.}
\label{fig:LLMresults}
\end{figure*}

\subsubsection{Data Selection for Image Classification}
For our classification experiments, our setting
is as follows: Given a target number of data 
points $k$ that we need to sample, we first train an initial model $\cal M$
using a uniformly random subset of $k'<k$ data points, and then
sample the remaining $k-k'$ data points using Algorithm~\ref{alg:oneRound} with the
loss defined by $\cal M$. Finally, we train a model 
on all the $k$
data points and evaluate it on a validation set. For our experiments, we chose $k'=0.2k$. We use three classic datasets from UCI, MNIST~\cite{lecun1998gradient}, 
FMNIST~\cite{xiao2017fashion}, and CIFAR-10~\cite{krizhevsky2009learning}. We give more details, and a comparison with other algorithm (including Margin and Entropy sampling) in \Cref{app:classification}.

\paragraph{Our algorithms.} We consider two instantiations of 
the sensitivity sampling algorithm presented in 
Algorithm~\ref{alg:oneRound}: \emph{loss-based sampling} and
\emph{gradient-based sampling}.
For loss-based sampling, we set 
            $\ell(e) := L(y, \text{model}_{\theta}(e))$ to be the
            \emph{loss} of the model on example $e$ with respect to the true label $y$, where $\theta$ are the model parameters. 
For gradient-based sampling, we set 
            $\ell(e) := \left\|\nabla_{\theta} L(y, \text{model}_{\theta}(e))\right\|_2^2$ to be equal to the squared $\ell_2$ norm of the gradient update. 
          
We present our results in Table~\ref{tab:lg}, Figure~\ref{fig:images} and Table~\ref{table:images}.
We notice that the loss- and gradient-based sampling
versions of \cref{alg:oneRound} perform best when 
the number of samples is relatively small. 
In addition, based on
the runtime comparison in Figure \ref{fig:runtimes}, the loss-based algorithm performs best in terms 
of runtime.

\begin{figure}[H]
\centering
      \begin{tabular}{lccc}
        \toprule
        %\multicolumn{2}{c}{Part}                   \\
        \cmidrule(r){1-2}
        Algorithm     & MNIST & Fashion MNIST & CIFAR10 \\
        \midrule
        uniform& $0.9130$ & $0.8091$ & $0.4587$ \\
        coreset [SS18] & $0.9134$ & $0.7692$ & $0.4491$\\
        \midrule
        Loss Alg~\ref{alg:oneRound} & $0.9203$ & ${\bf 0.8140}$ & $0.4590$ \\
        Grad Alg~\ref{alg:oneRound} & ${\bf 0.9207}$ & $0.8107$ & ${\bf 0.4598}$ \\
        \bottomrule
      \end{tabular}
    \caption{Experimental results for selecting $k=2000$ data points and different datasets. For each algorithm, we show the accuracy on the validation dataset.}\label{tab:lg}
\end{figure}

% \begin{figure*}[h]
% \begin{subfigure}{}%215pt}
\begin{figure}[H]
\centering
%\framebox[4.0in]{$\;$}
%\fbox{\rule[-.5cm]{0cm}{4cm} \rule[-.5cm]{4cm}{0cm}}
\includegraphics[width=0.49\linewidth]{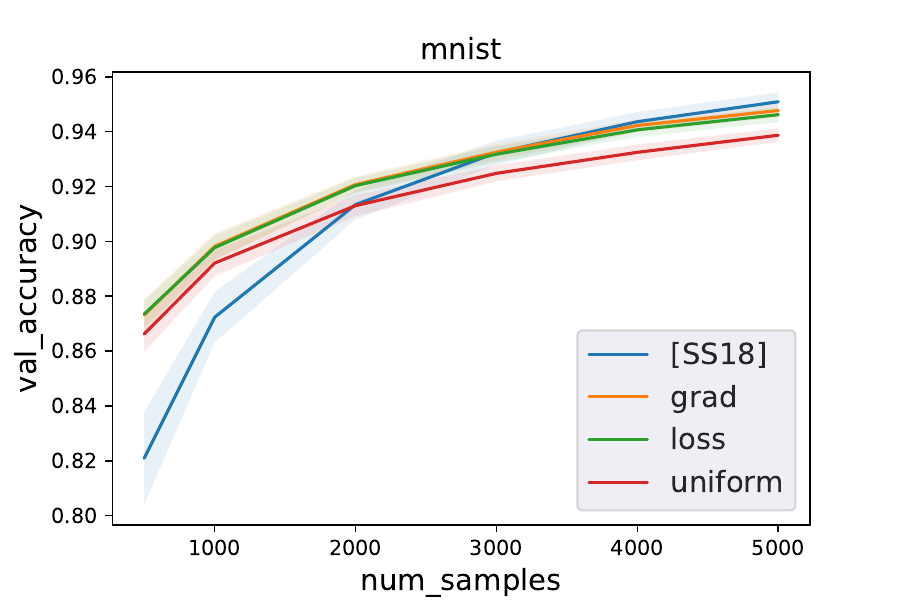}
\includegraphics[width=0.49\linewidth]{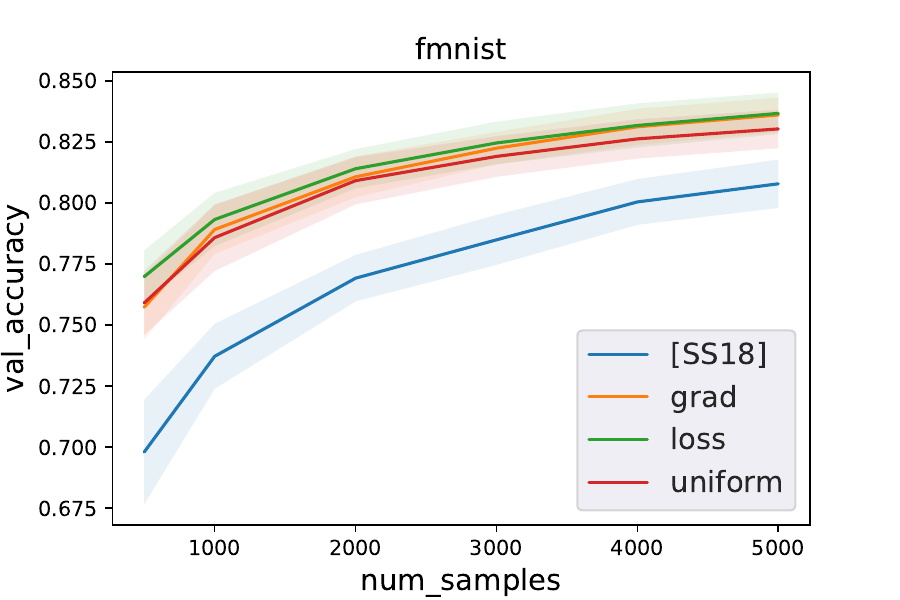}
\includegraphics[width=0.49\linewidth]{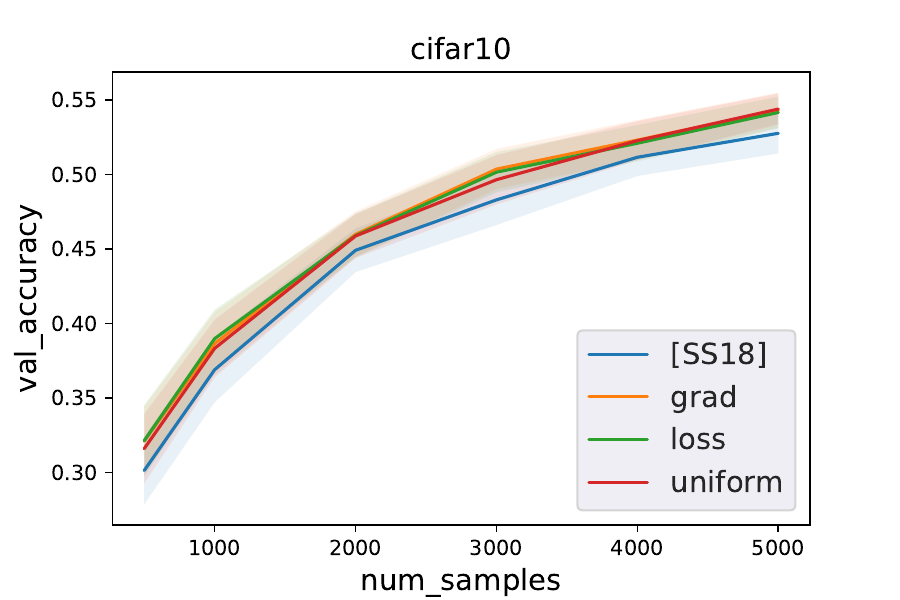}
    \caption{Plots of experimental results for different datasets. For each algorithm, we plot the accuracy on the validation dataset for different values of $k$ (number of
    samples). We also provide a runtime comparison on CIFAR10. We independently run each data point $100$ times, and present the mean with bands of one standard deviation.}
    \label{fig:images}
\end{figure}
% \end{subfigure}

% \begin{subfigure}{215pt}
\begin{figure}[H]
\centering
      \begin{tabular}{lccc}
        \toprule
        %\multicolumn{2}{c}{Part}                   \\
        \cmidrule(r){1-2}
        Algorithm     & MNIST & Fashion MNIST & CIFAR10 \\
        \midrule
        uniform coreset& $0.9130$ & $0.8091$ & $0.4587$ \\
        ~\cite{SenerS18} & $0.9134$ & $0.7692$ & $0.4491$\\
        \midrule
        Loss-based Algorithm~\ref{alg:oneRound} & $0.9203$ & ${\bf 0.8140}$ & $0.4590$ \\
        Gradient-based Algorithm~\ref{alg:oneRound} & ${\bf 0.9207}$ & $0.8107$ & ${\bf 0.4598}$ \\
        \bottomrule
      \end{tabular}
    \caption{Experimental results for selecting $k=2000$ data points and different datasets. For each algorithm, we show the accuracy on the validation dataset, averaged over 100 runs.}
    \label{table:images}
\end{figure}

% \end{subfigure}
% \caption{The experimental comparison. To minimize variation, we independently run each data point $100$ times, and present the mean with bands of one standard deviation.}
% \label{fig:results}
% \end{figure*}

\subsection{Experiments on Linear Regression}
\label{sec:experimental_results_regression}

Following our theoretical analysis in Section~\ref{sec:regression},
we validate our coreset sampling algorithm on a linear regression 
task. We present our results on the UCI gas sensor dataset from the University of California, Irvine repository~\cite{misc_gas_sensor_array_drift_dataset_224,vergara2012chemical,rodriguez2014calibration}
in Figure~\ref{fig:regression_results}. The dataset consists of 13910 input points
in 16 dimensions. 
We report the $R^2$ score, $R^2 := 1 - \frac{\sum_{i=1}^n (b_i - x_i)^2}{\sum_{i=1}^n (b_i - \overline{y})^2}$, where $\overline{y} := \frac{1}{n} \sum_{i=1}^n b_i$. We run \cref{alg:regression} with
some implementation details that are deferred to
Appendix~\ref{app:regression_algo_details}

We compare with uniform sampling and leverage score sampling.
Leverage score sampling is the standard sampling algorithm for linear
regression, and is known to have extremely good performance but
high runtime cost, since it requires solving a full-dimensional
linear system per sample data point. Surprisingly, we find that our
clustering-based algorithm performs almost equally well as leverage 
score sampling, while being drastically faster -- the $k$-medoids solution can be computed in linear time. 
We report the observed value for $\Lambda$ across the whole dataset in Table~\ref{table:lambda}.

\begin{figure*}[tb]
\centering
\includegraphics[width=0.4\linewidth]{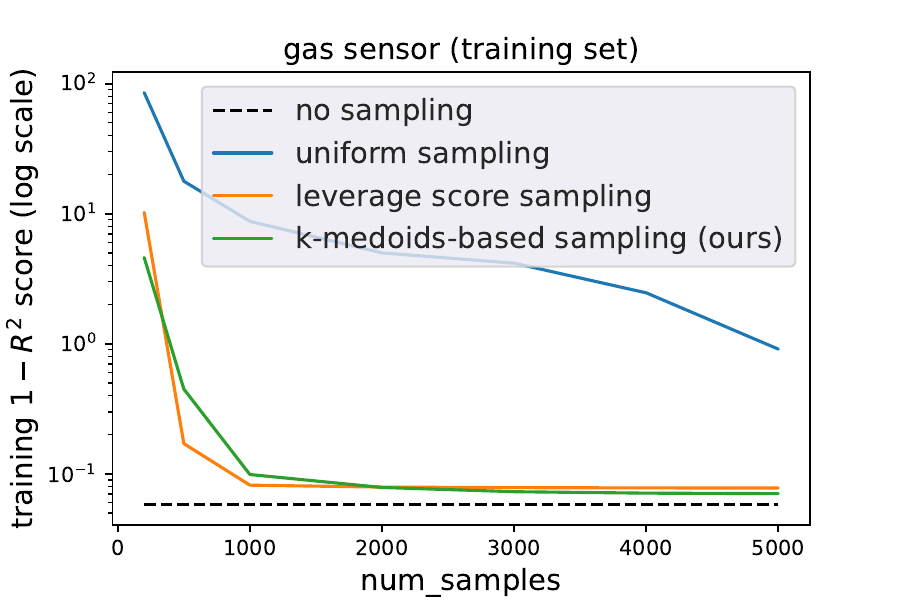}
\includegraphics[width=0.4\linewidth]{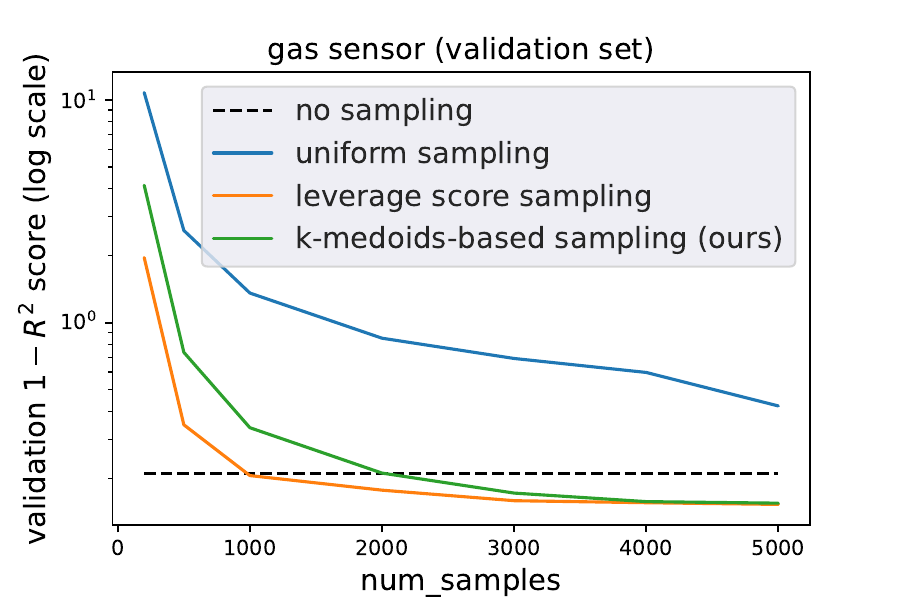}
\caption{Experimental results on the gas sensor regression dataset.
Each data point is the average of $100$ runs.}
\label{fig:regression_results}
\end{figure*}
\begin{comment}
We consider the problem of data selection, in which the goal is to train a model
on a new dataset in a runtime-efficient and data-efficient way. Runtime efficiency means that we prefer an algorithm with low
total runtime and a
\emph{sublinear} number of inference queries to the model.
Data efficiency means that there is a bound on the number of data 
points that can be used for training because labels are 
costly or hard to obtain, or the model is costly to train on a large dataset. 
\end{comment}

\newpage

\bibliography{sample}
\bibliographystyle{icml2024}

\newpage

\onecolumn
\appendix
\section{Preliminaries}

\subsection{Further Related Work}
\label{appendix:further_related_work}

% \todo{All those references are somewhat old, there's nothing about papers who cite/improve Sener and Savarese}

We present related work in data selection and active learning. For extensive and recent survey, we refer to \cite{ren2021survey}, and detail here some of the most relevant point for comparison.

Our work departs from previous work in the following ways: We work
in the regime where we have a budget on both the number of elements
selected (i.e.: labeled) and the number of inferences of the model.
Moreover, we present a rigorous analysis of our sampling mechanism
that works for a variety of machine learning models, as long as 
we are provided with embeddings that are Lipshitz with respect to 
the loss of the model.

Given an unlabeled set of points, the active learning problem
asks to identify the most relevant points to label~(\cite{settles2009active,cohn1996active}). In the era of big
data, labeling a big dataset is often too expensive. We are thus
given a budget of, say, $k$ elements that we can label. The 
question becomes how to pick these $k$ elements so as to maximize
the performance of the final model (that will be trained on these
elements). 

From a theoretical standpoint, \cite{dasgupta2004analysis} demonstrated that greedy active learning heuristics perform
poorly if agnostic to both data and learning algorithm. 
To circumvent these negative results, 
other works have made assumptions on the data-dependent realizability of the hypothesis space like
(\cite{pmlr-v28-gonen13}) or on a data dependent measure of the concept space called disagreement coefficient
(\cite{hanneke2007bound}).

Related to ours, several successful works have brought together unsupervised techniques 
such as clustering and information from the model (such as margin
scores), see e.g.:~\cite{citovsky2021batch}. The work of~\cite{SenerS18} brings together clustering, and sketching
techniques (coresets in this case).

Another line of works consists of bayesian active learning methods which use a non-parametric model, like a Gaussian process, to
obtain an estimate of the expected improvement on the model 
after each query (\cite{kapoor2007active}), or alternatively the 
expected error after a set of queries (\cite{roy2001toward}).
It seems that an important drawback is that these approaches
do not scale to large models (see the discussion in~\cite{SenerS18}).

Uncertainty based methods form another important family of active learning algorithm. They aim at finding relevant examples using
heuristics like highest entropy (\cite{joshi2009multi}), or geometric distance to decision boundaries
(\cite{tong2001support,brinker2003incorporating}).

Batch-active learning based on uncertainty may lead to a useless batch, where all queries are very similar (when the highest uncertainty is concentrated in a small region). To cope with this,  several methods that aim at trading-off diversity 
and uncertainty to select the points. The way the elements are iteratively selected can vary depending 
on the application from mini-batches to one-shot (e.g.:~\cite{hoi2006batch,guo2007discriminative,chakraborty2014adaptive,citovsky2021batch,amin2020understanding})
Specifically in the context of mini-batch active learning, a 
common approach is to use unsupervised machine learning techniques
to extract information from the data. Such methods include $k$-
Medoid~(\cite{schubert2019faster}), or 
MaxCover~(\cite{hochbaum1998analysis}) to select a set of data 
points that maximally cover the dataset with respect to some 
objective. 
\cite{elhamifar2013convex} and \cite{yang2015multi} design a discrete optimization problem for this purpose, that they solve using convex optimization
methods. Unfortunately, the running time of these methods
is quadratic in the input data size and so highly impractical for large data.

Covering or clustering approaches have also been tried in the past~\cite{joshi2010multi,wang2015querying}. The former does not
provide any theoretical guarantee associated to its approach.
 The latter uses empirical risk minimization  to minimize the difference between the maximum mean
discrepancy between iid. samples from the dataset and the actively selected samples (instead of the loss we work with).

Active learning has also been studied when tailored to some specific machine learning
problems such as nearest neighbors, logistic regression or 
linear regression with Gaussian noise  (\cite{wei2015submodularity,hoi2006batch,guo2007discriminative,yu2006active}).

Recently, active learning was also extended to $\ell_p$-regression for all $p \geq 1$ without any assumptions on the data, resulting in a number of optimal bounds
\cite{CP19,CD21,PPP21,MMW022,WY23}. These works are based on using sampling probabilities defined from the design matrix (agnostic to the label vector), and range from leverage scores $(p = 2)$ to $\ell_p$-sensitivities to $\ell_p$-Lewis weights, the latter achieving optimal bounds for all $p \geq 1$. Our work adds a new set of scores to this growing literature for regression, namely, scores that are proportional to the cost of clustering individual points. We note that in practice, computing an approximate $k$-median solution may be much faster than approximating the Lewis weights or leverage scores of a matrix since it does not involve computing the inverse of any matrices.  

If the model can be run on all input data, and so the confidence
of the model is known for all the input data elements, then
several set-cover based methods that aims at best covering the 
hypothesis space have been designed (\cite{guillory2010interactive,golovin2011adaptive,EsfandiariKM21}. 
The key distinguishing factor of our approach compared to these
works is that we do not require the model to be run on all the 
input data. Furthermore, as we demonstrate, our sampling technique applies more generally to problems other than regression.

\paragraph{Coresets}
Our ideas are inspired from the coreset literature. Coreset were introduced initially for $k$-median and $k$-means clustering: the goal is to compute a (weighted) set $S$ of points such that, for any set of $k$ centers, evaluating its cost on $S$ is almost the same as evaluating it on the full dataset \cite{HaM04}. Coresets with optimal size (which is $O(k\eps^{-2} \min(\eps^{-2}, \sqrt{k})$) exist \cite{CSS21, Cohen-AddadLSSS22, HLW23}, and one of the most standard tool to build a coreset is sensitivity sampling, namely sampling according to the cost in a constant-factor $(k, z)$-clustering solution. 

Ideas from the  literature on coreset for clustering have already spread to other domains: \cite{icml/TukanZMRBF23} presents coresets for Radial basis function neural networks of small sizes, and \cite{NEURIPS2020_0afe095e} for near-convex functions.
\cite{Mussay2020Data-Independent,DBLP:conf/nips/TukanMM22}  showed how to use coreset for prunning and compressiong neural networks, and \cite{MaaloufJF22} used coreset for fast least-square linear regression. For machine-learning tasks, 
\cite{DBLP:conf/tamc/TukanBFR20} present coresets for Support Vector Machines.

% Importance sampling for active learning (\cite{ganti2012upal}) 
% but not scalable.

\subsection{Clustering Preliminaries}
\label{appendix:clustering_preliminaries}

In the following, we are given a set of points
$\calD$ in the Euclidean Space $\R^d$ with $\ell_2$ norm. We let $\mu_z(\calD )$ be the power mean of $X$, namely the point $p$ that minimizes $\sum_{x \in \calD } \|x-p\|^z$. We let $\disp_z(\calD ) := \sum_{x\in X} \|x-\mu(\calD )\|^z$.

Given a set of $k$ points $C \in (\R^d)^k$, we denote the $(k, z)$-clustering cost of $C$ on $\calD $ as $\cost_z(\calD , C) := \sum_{x \in \calD } \min_{c \in C} \|x-c\|^z$
and $\cost_{k, z}(\calD ) := \min_{C \subset \R^d,~|C| \le k} \cost_z(\calD ,C)$. 
For $z=1$, this objective corresponds to $k$-median, while $k$-means is for $z=2$. Note that $\disp_z(\calD ) = \cost_{1, z}(\calD)$.

We say that a set $C$ of $k$ points is an $\alpha$-approximation to $(k, z)$-clustering on $\calD$ when $\cost_z(C, \calD) \leq \alpha \cost_{k, z}(\calD)$. 
An ordered list of centers $c_1, ..., c_n$ is an $\alpha$-approximation to Prefix-$z$-clustering when, for all $1 \le k \le n$, $(c_1, ..., c_k)$ is an $\alpha$-approximation to $(k,z)$-clustering on $\calD$.
An $O(1)$-approximation to prefix $(k,z)$-clustering can be computed using the algorithm of \cite{MettuP03}. $D^z$-sampling (which is $k$-means++ for $z=2$) gives an $O(\log k)$-approximation, which is fast and performs extremely well in practice.

\section{Deferred Proofs}
 In this section, we use Bernstein's concentration inequality: 
\begin{thm}[Bernstein's inequality]
\label{thm:bernstein}
    Let $X_1, \ldots, X_n$ be independent random variables and let $M > 0$ be 
    such that, for all $i$, $| X_i|\le M$. Then, for all $t > 0$,
    \[Pr[|\sum_{i} X_i - E[\sum_i X_i]| \ge t] \le \exp(-\frac{t^2}{2\sum_{x \in X} E[X_x^2] + 2Mt/3}) \]
\end{thm}

\subsection{Proof for the non-adaptive case}
\label{appendix:proof_non_adaptive}

\begin{proof}[Proof of \cref{thm:nonAdaptive}]
We denote for simplicity $R  = \sup_{e \in \calD} \ell(e)$. The upper-bound is a simple application of Bernstein's inequality and is included for completeness.

The algorithm chooses successively $s$ uniformly random samples $S_1, ..., S_s$ from $\calD$ (with replacement) and gives each sampled element weight $n/s$.
 % Let  be the successive random samples, and l
  Let $X_i = w(S_i) \ell(S_i)$. It holds that $\E[X_i]= \frac{n}{s} \sum_{e \in 
 \cal D} \frac{\ell(e)}{n} = \frac{\sum_{e \in \calD} \ell(e)}{s}$, and, thus, $\E[\sum_{e \in S} w(e) \ell(e)]= \E[\sum X_i] = \sum_{e \in \calD} \ell(e)$.

  To show that this sum of random variables is  concentrated, we aim at applying Bernstein's inequality. For this, we need to bound $\E[X_i^2]$: we have
  \begin{align*}
      \E[X_i^2] &= \sum_{e \in \calD} \left(\frac{n}{s}\ell(e)\right)^2 \Pr[e = S_i] 
      \leq \sum_{e \in \calD} \frac{n}{s^2} \ell(e)^2 
      \leq \frac{n^2}{s^2} R^2.  
      \end{align*}
Summed over all $i$, we therefore have $\sum_{i=1}^s \E[X_i^2] \leq n^2 R^2 / s$. Furthermore, for any $i$, $|X_i| \leq \frac{n}{s}R$ with probability 1.
Plugging this result into Bernstein's inequality yields
\begin{align*}
        \Pr[\Delta(S) \ge \eps n R] &=
        \Pr\left[\left|\sum_{i} X_i - E[\sum_i X_i]\right| \ge \eps n R\right] 
        \le \exp\left(-\frac{\eps^2 n^2 R^2 \cdot s}{2n^2 R^2+ 2nR \cdot \eps n R / 3}\right)\\
        &\leq \exp(-\eps^2 \cdot s / (2+\epsilon)).
\end{align*}

With $s = O(1/\eps^2)$, this gives the desired probability bound.

For the lower bound, consider a multiset $\calD \subset \R$ with $n/2$ copies of $-1$ and $n/2$ copies of $1$, and $\ell$ being the identify function $\ell(x) = x$, implying that $\sum_{e \in \calD} \ell(e) = 0$. Then, the estimator is a sum of Rademacher random variables, and anti-concentration bound states that for any fixed value $x$, $\sum_{s \in S} \ell(s) = x$ with probability at most $O(1/\sqrt{|S|})$ \cite{littlewood1939number}. Therefore, for some constant $c$ (which depends on the previous big-O constant), $\Pr[|\sum_{s \in S} \ell(s)| \geq c \sqrt{|S|}] \geq 1/2$.  Multiplying with $n/|S|$, this implies that our estimator is bigger than $\frac{c n}{\sqrt{|S|}}$ with probability at least 1/2. When $|S| \leq 1/\eps^2$, this gives the desired statement.
\end{proof}

\subsection{Proof of Theorem~\ref{thm:main}}
\label{appendix:proof_thm_mainPrecise}

\begin{proof}
    We prove that \cref{alg:oneRound} yields the 
    desired result. Let $S_1, ..., S_s$ be the successive random samples, and define $X_i = \ell(S_i) w(S_i)$. 

    % The Bernstein inequality (see \cref{thm:bernstein}) implies, for $t = \sqrt{2s \sum_{i=1}^s \E[X_i^2]} + 2 \eps Ms/3$,
    % \begin{align*}
    % &Pr\left[\Delta(S) > \eps t\right] \\
    % \le 
    % &\exp\left(\frac{-\eps^2 t^2}{
    % 2\sum_{i=1}^s \E[X_i^2] +  2 M \eps t /3}\right)\\
    % \le & \exp\left(-\eps^2 s\right).
    % \end{align*}

    % With our choice of $s$, this is at least $??$ \ds{formalize}. 

     We combine Bernstein's inequality (see \cref{thm:bernstein}) with \cref{ass:lip} to bound the value of $t$. Recall that $\cost^\Lambda_{\calC, z}(\calD) := \sum_{i=1}^k \Lambda_i \cost_z(C_i, c_i)$; and when $e \in C_i$ we let $\Lambda_e = \Lambda_i$
    
    From \cref{ass:lip}, we get that for $e \in C_i$, $\ell(e) \leq \hat{\ell}(e) + \Lambda_i v(e)$ and $\hat{\ell}(e) \leq \ell(e) + \Lambda_i v(e)$. Therefore, 
    \begin{align*}
        E[X_i^2] &= \sum_{e \in \calD} (\ell(e) w(e))^2 \cdot p_e = \sum_{e \in \calD} \ell(e)^2 \frac{1}{p_e s^2} \\
        &= \frac{1}{s^2}\cdot \sum_{e \in \calD} \ell(e)^2 \cdot \frac{\cost^\Lambda_{\calC, z}(\calD) + \sum_{x \in \calD} \hat{\ell}(x)}{\hat{\ell}(e) + \Lambda_e v(e)}\\
        &\leq \frac{1}{s^2}\cdot \sum_{e \in \calD} \ell(e) \cdot (\hat{\ell}(e) + \Lambda_e v(e)) \cdot \frac{\cost^\Lambda_{\calC, z}(\calD) + \sum_{x \in \calD} \hat{\ell}(x)}{\hat{\ell}(e) + \Lambda_e v(e)}\\
        &= \frac{1}{s^2}\cdot \sum_{e \in \calD} \ell(e)  \cdot \left(\cost^\Lambda_{\calC, z}(\calD) + \sum_{x \in \calD} \hat{\ell}(x)\right)\\
        &\leq \frac{1}{s^2}\cdot \left(\sum_{e \in \calD} \ell(e) \right) \cdot \left(2\cost^\Lambda_{\calC, z}(\calD) + \sum_{x \in \calD} \ell(x)\right)\\
        &\leq \frac{1}{s^2}\cdot \left(2\cost^\Lambda_{\calC, z}(\calD)+ \sum_{e \in \calD} \ell(e) \right)^2.       
    \end{align*}
    We also let $M := \max_{e \in \calD} \ell(e) w(e) = \max_e \ell(e) \cdot \frac{\cost^\Lambda_{\calC, z}(\calD) + \sum_x \hat{\ell}(x)}{s(\hat{\ell}(e) + \Lambda_e v(e))}$.
    Similarly,  we have $M \le \frac{1}{s}\left(\cost^\Lambda_{\calC, z}(\calD) + \sum_{e \in \calD} \hat \ell(e) \right) \le  \frac{1}{s}\cdot\left( 2 \cost^\Lambda_{\calC, z}(\calD) + \sum_{e\in \calD} \ell(e) \right)$.

    % Therefore, we get that with probability $99/100$, $t \leq 4 \cost^\Lambda_{\calC, z}(\calD) + 2\sum_{e\in \calD} \ell(e)$: by union-bound, we conclude that 
    Applying Berstein's inequality with $t = 2 \cost^\Lambda_{\calC, z}(\calD) + \sum_{e\in \calD} \ell(e)$ therefore yields:
    \begin{align*}
    %     &Pr\left[\Delta(S) > 2\eps \left( 2 \cost^\Lambda_{\calC, z}(\calD) + \sum_{e\in \calD} \ell(e)\right)\right] \\
    % \le 
    % &\exp(-1).
    %
    \Pr\left[\Delta(S) > \eps t\right] &\le 
    \exp\left(-\eps^2 \frac{t^2 \cdot s}{
    2t^2 + 2\eps t^2 /3}\right)
    = \exp\left(-\eps^2 \frac{s}{
    2 + 2\eps /3}\right) \le \exp(-1),\\
    %&\le \exp(-\eps^2 \cdot s)%\\
    % &\le \exp(-\eps^2 \frac{(10\lambda \cost(X) + \sum_{x} \ell(x))^2}{
    % \eps^{-2} (\sum_{x} \ell(x) + 2\lambda \cost(X)) \sum_{x} \ell(x) })\\
    % &\le \exp(-\eps^2 
    % \frac{(10\lambda \cost(X) + \sum_{x} \ell(x))^2}{
    % \eps^{-2} (\sum_{x} \ell(x))^2 + 2\lambda \cost(X) \sum_{x} \ell(x) })\\
    % &\le \exp(-2),
    \end{align*}
    where the last inequality holds by the choice of $s = \lceil\eps^{-2} (2+2\eps/3)\rceil$. 
    % Since $s \geq 1/\eps^2$, this gives constant probability as desired.
\end{proof}

\subsection{Adaptive Active learning}
\label{app:rounds}

We present here the algorithm used to prove \Cref{thm:rounds}. The proof follows directly from the proof for \Cref{thm:main}.
\begin{algorithm}
\caption{Adaptive-active-learning($\calD, r, \lambda, \eps$)}
\label{alg:rRound}
\begin{algorithmic}[1]
\STATE Compute a $O(1)$-approximation to prefix-$z$-clustering on $\calD$, i.e., an
ordering of the points in $\calD$ such that any prefix of length $k_0$ is an
$O(1)$-approximation to $(k_0, z)$-clustering on $\calD$.
\FOR{each round $i = 1, ..., r$}
\STATE  query $\ell$ for the points at position in $[(i-1)k + 1, ik]$
in the ordering, and 
define $\calA$ to be the set of elements at position at most $ik$ in the 
ordering.
\STATE For $e \in \calD$, define $\calA(e) = \argmin_{a \in \calA} \|e-a\|$ the element of $\calA$ 
    that is the closest to $x$,  $\hat{\ell}(e) := \ell(\calA(e))$ and 
    $v(e) := \|e-\calA(x)\|$.
\STATE Define $p_e := \frac{\hat{\ell}(e) + \lambda v(e)}{\lambda \cost_z(\calD, \calA) + \sum_{x \in \calD} \hat{\ell}(x)}$.
\STATE Let $s:= \lceil\eps^{-2} (2+2\eps/3)\rceil$. For $e \in \calD$ define $p_e := \frac{\hat{\ell}(e) + \lambda v(e)}{\lambda \cost_z(\calD, \calA) + \sum_{x \in \calD} \hat{\ell}(x)}$ and $w_i(e) = s^{-1} p_e^{-1}$.
\STATE Compute a sample $S_i$ of $s$ points, picked independently following the distribution $p_e$. 
\ENDFOR
\STATE \textbf{Output:} $S_i, w_i$ for each round $i$.
\end{algorithmic}
\end{algorithm}

\subsection{The parameter $\Lambda$}\label{app:lambda}
\begin{proof}[Proof of \Cref{lem:upperBoundLambda}]
    We have, for $t$ points $x_1, ..., x_t$ uniformly at random in $C_i$:
    \begin{align*}
        \Pr\left[\max_{j} \frac{|\ell(x_j)-\ell(c_i)|}{\|x_j-c_i\|^z} \in [\Lambda_i/\log(n), \Lambda_i]\right] \geq 1-(1-p)^t.
    \end{align*}
    With $t = \log(100 k) / \log(1-p)$, the maximum of the estimate $\frac{|\ell(x_j)-\ell(c_i)|}{\|x_j-c_i\|^z}$ is in $[\Lambda_i/\log(n), \Lambda_i]$ with probability at least $1-1/(100k)$. Multiplying $\max_{j} \frac{|\ell(x_j)-\ell(c_i)|}{\|x_j-c_i\|^z} $ by $\log(n)$ thus yields our upper bound on $\Lambda_i$.
\end{proof}

\subsection{Regression}\label{app:regression}

To show \cref{thm:regression}, we first prove that for any fixed $x \in \calX$, the desired bound hold with probability $1-\exp(-\eps^2 s)$. It is standard to extend this result to all $x \in \calX$, using discretization techniques to find a set $N$ of size $\eps^{-O(d)}$ such that preserving the cost for all vectors in $N$ is enough to extend the result for all candidate $x$ ($N$ is called a net for $\calX$). Hence, it is enough to show the following lemma:

\begin{lem}\label{lem:regression}
Let $\Lambda \in \R^k$ with $\lambda_i \geq 1$, and $A$ and $b$ that respects \cref{ass:regression} with constant $\zeta$ and $\hat a_i$ and $x_0$ as computed by \cref{alg:regression}. 

    Let $x\in \calX$, namely $x$ such that $\|x\|_2 = O(1)$ and there is some $\zeta \geq 1$ such that $\forall j \in C_i, |\langle \hat a_j, x-x_0\rangle | \leq \Lambda_i \|a_j - \hat a_j\|_2$.
    Then, with probability $1-\delta$, it holds that for $s = 8\eps^{-2} \log(1/\delta)$,
    $$\left|\sum_{s \in S} w(s) (\langle a_s, x\rangle  - b_s)^2 - \|Ax-b\|_2^2\right| \leq \eps (\|Ax-b\|_2^2 + \sum_{i \in [k]} \Lambda_i \cost_{1, 1}(C_i)))$$
\end{lem}
\begin{proof}
    Let $S_1, ..., S_s$ be the successive random samples, and define $X_t = w(S_t) (\langle a_{S_t}, x\rangle  - b_{S_t})^2$.
    By choice of $w$, it holds that $\E[\sum X_t] = \|Ax-b\|_2^2$. We will show concentration using the Bernstein inequality. 

    We denote for simplicity $\cost_\Lambda(A) = \sum_{i \in [k]} \Lambda_i \cost_{1,1}(C_i)$. For $j \in C_i$, we let $\tilde \Lambda_j = \Lambda_i$.
    
    We first focus on bounding the second moment of $X_t$. We have:
    \begin{align*}
        \E[X_t^2] &= \sum_{i=1}^n \frac{(\langle a_i, x\rangle  - b_i)^4}{s^2 p_i}\\
        &=  \frac{1}{s^2} \cdot \sum_{i=1}^n (\langle a_i, x\rangle  - b_i)^4 
        \cdot \frac{\sum_{j \in [n]} \tilde \Lambda_j \|a_j - \hat a_j\| + v(a_{j}, x_0)}{\tilde \Lambda_i \|a_i - \hat a_i\|  + v(a_i, x_0)}
    \end{align*}

    To bound this term, we first note that
    \begin{align}
        \notag
        &(\langle a_i, x\rangle  - b_i)^2 - (\langle \hat a_i, x\rangle  - \hat b_i)^2 \\
        \notag 
        = & (\langle a_i + \hat a_i, x\rangle  - b_i - \hat b_i)(\langle a_i - \hat a_i, x\rangle  - b_i + \hat b_i)\\
        \notag 
        = &\langle a_i + \hat a_i, x\rangle \langle a_i - \hat a_i, x\rangle + \langle a_i + \hat a_i, x\rangle  \cdot (\hat b_i - b_i) - (b_i + \hat b_i) \langle a_i - \hat a_i, x\rangle  - (\hat b_i + b_i) (\hat b_i - b_i)\\
        \notag
        \leq &\|a_i + \hat a_i\| \cdot \|a_i - \hat a_i\| \cdot \|x\|^2 + \|a_i + \hat a_i\| \cdot \| x\| |\hat b_i - b_i| - |b_i + \hat b_i| \cdot \|a_i - \hat a_i\|\cdot \|x\| + |b_i + \hat{b}_i| \cdot |b_i - \hat{b}_i| \\
        \label{eq:atocenter}
        =& O(\tilde \Lambda_i \|a_i - \hat a_i\|),
    \end{align}
    where the last two lines follow from Cauchy-Schwarz and \cref{ass:regression}.
    
We now relate this to the term $v(a_i, x_0) = (\langle \hat a_i, x_0\rangle  - \hat b_i)^2$ of the denominator:
\begin{align*}
    (\langle \hat{a}_i, x\rangle  - \hat{b}_i)^2 - (\langle \hat{a_i}, x_0\rangle  - \hat{b}_i)^2 &= \langle \hat a_i, x-x_0 \rangle \cdot ( \hat a_i, x-x_0 \rangle + 2\hat b_i)\\
&= \langle \hat{a}_i, x + x_0\rangle \langle \hat{a}_i, x - x_0\rangle  - 2\hat{b}_i \langle \hat{a}_i, x-x_0\rangle \\
&=O(|\langle \hat{a_i}, x-x_0\rangle|) = O(\tilde \Lambda_i \|a_i - \hat a_i\|),
\end{align*}
where the last line uses our assumption $|\langle \hat{a_i}, x-x_0\rangle| \leq \tilde \Lambda_i \|a_i - \hat a_i\|$.
    Thus, combining those equations:
    \begin{align*}
        (\langle a_i, x\rangle  - b_i)^2  &\leq O(\tilde \Lambda_i \|a_i - \hat a_i\|) + (\langle \hat{a_i}, x_0\rangle  - \hat{b}_i)^2 + O(\tilde \Lambda_i \|a_i - \hat a_i\|)\\
&= O(\tilde \Lambda_i \|a_i - \hat a_i\| + v(a_i, x_0)).
    \end{align*}

    Thus, we can now finish our bound on the second moment of $X_t$: 
    \begin{align*}
     \E[X_t^2] &=  \frac{1}{s^2} \cdot \sum_{i=1}^n (\langle a_i, x\rangle  - b_i)^4 \cdot \frac{ \sum_{j \in [n]} \tilde \Lambda_i \|a_j - \hat a_j\| + v(a_j, x_0)}{\tilde \Lambda_i \|a_i - \hat a_i\|  + v(a_i, x_0)}\\
    &\leq \frac{1}{s^2} \cdot \sum_{i=1}^n (\langle a_i, x\rangle  - b_i)^2 \cdot  O(\tilde \Lambda_i) \|a_i - \hat a_i\| + v(a_i, x_0)) \cdot \frac{ + \sum_{j \in [n]} \tilde \Lambda_j \|a_{j} - \calA(a_{j})\| + v(a_{j})}{\tilde \Lambda_i \|a_i - \hat a_i\|  + v(a_i, x_0)}\\
     &\leq \frac{\sum_{j \in [n]} \tilde \Lambda_j \|a_j - \calA(a_j)\| + v(a_j, x_0)}{s^2} \cdot \sum_{i=1}^n O\left((\langle a_i, x\rangle  - b_i)^2\right)
    \end{align*}

    Using the same upper bounds, we get that for all $i$,
    \begin{align*}
        w(i) (\langle a_{i}, x\rangle  - b_{i})^2 &= \frac{1}{s} \cdot (\langle a_i, x\rangle  - b_i)^2 \cdot \frac{\sum_{j \in [n]} \tilde \Lambda_j \|a_{j} - \calA(a_{j})\| + v(a_{j}, x_0)}{\tilde \Lambda_i \|a_i - \hat a_i\|  + v(a_i, x_0)}\\
        &\leq \frac{\sum_{j \in [n]} \tilde \Lambda_j \|a_{j} - \calA(a_{j})\| + v(a_{j}, x_0)}{s}.
    \end{align*}

    Therefore, for $T =  \eps (\|Ax-b\|_2^2 + \cost_\Lambda(A) + \sum_{i} v(a_i, x_0))$ we get that $2\sum_t \E[X_t^2] \leq T^2 / s$, and with probability $1$ each $X_t$ is verifies $2|X_t|T/3 \leq T^2/s$.
    Furthermore, using \cref{eq:atocenter} and the optimality of $x_0$ for the dataset $\{\hat a_1, ..., \hat a_n\}$, we have that: 
    \begin{align*}
        \sum_i v(a_i, x_0) &= \sum_i (\langle \hat a_i, x_0\rangle - \hat b_i)^2 \leq \sum_i (\langle \hat a_i, x\rangle - \hat b_i)^2\\
        &\leq  \sum_i (\langle  a_i, x\rangle - b_i)^2 + O(\tilde \Lambda_i) \|a_i - \hat a_i\| = \|Ax-b\|_2^2 + O(\cost_\Lambda(A)).
    \end{align*}
    Hence, the Bernstein inequality ensures that
    \begin{align*}
        &\Pr\left[\left|\sum_{e \in S} w(e) (\langle a_e, x\rangle  - b_s)^2 - \|Ax-b\|_2^2\right| \geq \eps (\|Ax-b\|_2^2 + \cost_\Lambda(A)))\right] \\
        \leq &\Pr\left[\left|\sum_{e \in S} w(e) (\langle a_e, x\rangle  - b_e)^2 - \|Ax-b\|_2^2\right| \geq \eps/2 \cdot (\|Ax-b\|_2^2 + \cost_\Lambda(A) + \sum_{i} v(a_i, x_0))\right] \\
        \leq &\exp\left(-\eps^2 s / 8\right).
    \end{align*}
    Therefore, using that $s = 8\eps^{-2} \log(1/\delta)$ concludes the lemma.
\end{proof}

\section{More experimental details}

\subsection{Fine-Tuning LLMs}

\paragraph{Model Details.}\label{apndx:llm-model-details} For our fine-tuning task, we display the model hyperparameters in the table below. We used a batch size of 128, a constant learning rate of 0.001, and dropout of 0.1. We fine-tune from the T5-Small \cite{t5paper} pre-trained model that was pre-trained for 1M steps. 

\begin{table}[ht]
\centering
\begin{footnotesize}
\begin{tabular}{ccc}
\textbf{Model Name} & \textbf{T5-Small} \\
\toprule
embedding dim & 512 \\
number of heads & 6 \\
enc./dec. layers & 8 \\
head dim & 64 \\
mlp dimension & 1024 \\
\midrule
number of parameters & 77M \\\bottomrule
\end{tabular}
\end{footnotesize}
\caption{Dimensions of T5 Small}
\label{tab:models_dim1}
\end{table}

\paragraph{Hölder Constant.}\label{apndx:llm-holders} We experimented with varying the Hölder constant from 0.05 to 500. Each run used BERT embeddings for clustering (as above), and sampled roughly 1\% of the whole dataset. We found that raising the constant too high (e.g., 500) resulted in a drop in quality and final accuracy more equivalent with diversity sampling (figure \ref{fig:en-de-lipshitz}).

\begin{figure}
\includegraphics[width=0.49\linewidth]{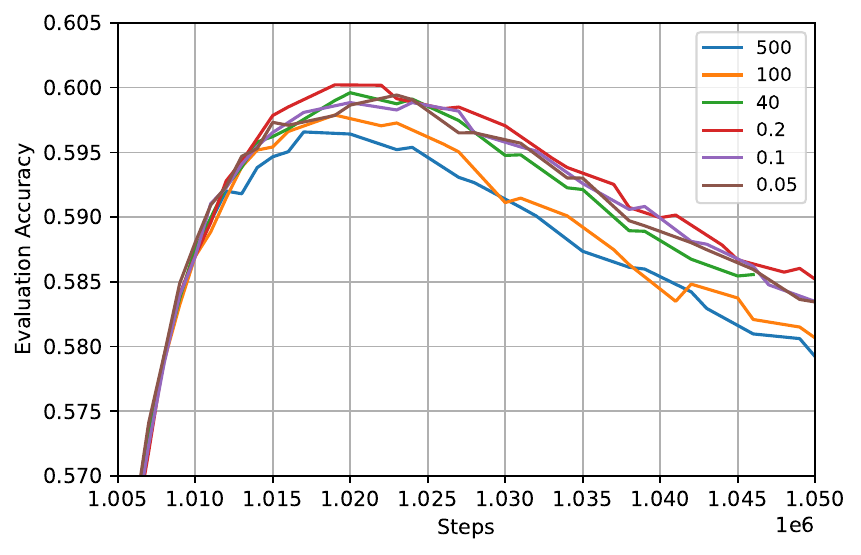}
\includegraphics[width=0.49\linewidth]{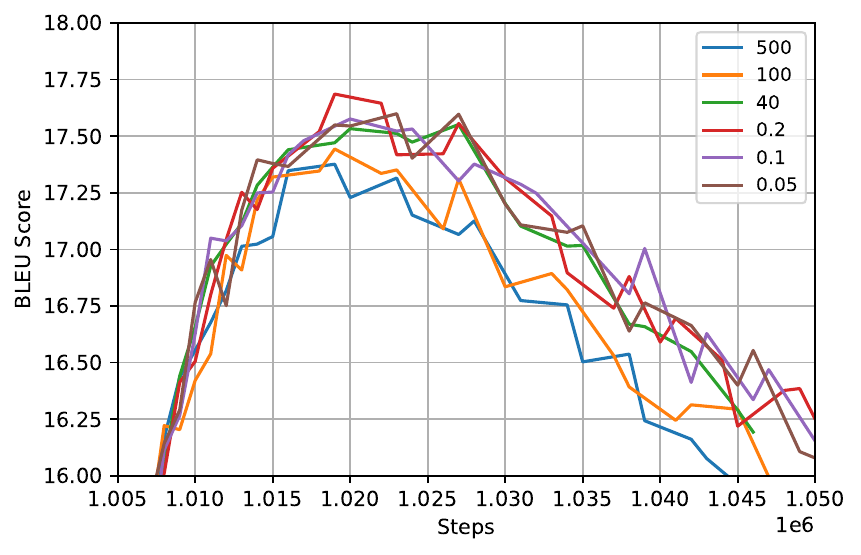}
\caption{We 
report the accuracy (left) and BLEU score (right) with differing Hölder constants from 0.05 to 500. Each run uses roughly 1\% of the whole
dataset.}
\label{fig:en-de-lipshitz}
\end{figure}

\subsection{Classification Tasks}\label{app:classification}

\paragraph{Datasets and models.} We largely follow the dataset and model setup used in DISTIL. The datasets
we consider are MNIST, Fashion MNIST and CIFAR10. For the first
two we use a neural network with one $128$-dimensional hidden layer, and for
the last one we use convolutional neural network with three convolutional 
layers and three dense layers. We train each model for $10$ epochs, a 
batch size of $32$, and use Adam optimizer with a learning rate of $10^{-3}$.

\paragraph{Runtime comparison.}

In Figure~\ref{fig:runtimes}, we show a comparison between the runtimes
of our loss- and gradient-based algorithms, and the coreset algorithm
of~\cite{SenerS18}. All algorithms were implemented in python using the tensorflow framework 
and the runtime calculation experiments
ran on CPU, on a cloud VM with 24 CPUs and 100GB of RAM.
It should be noted that a significant advantage of our loss- and
gradient-based sampling is that they can rely on a pre-computed
metric and clustering that is not updated during the sampling
process. In most applications, this will be a fixed metric
generated by an upstream model, that is easy to generate and
compute distances. As a result, most of the runtime will be spent
running model inferences at the cluster center points.

\begin{figure}
\includegraphics[width=0.49\linewidth]{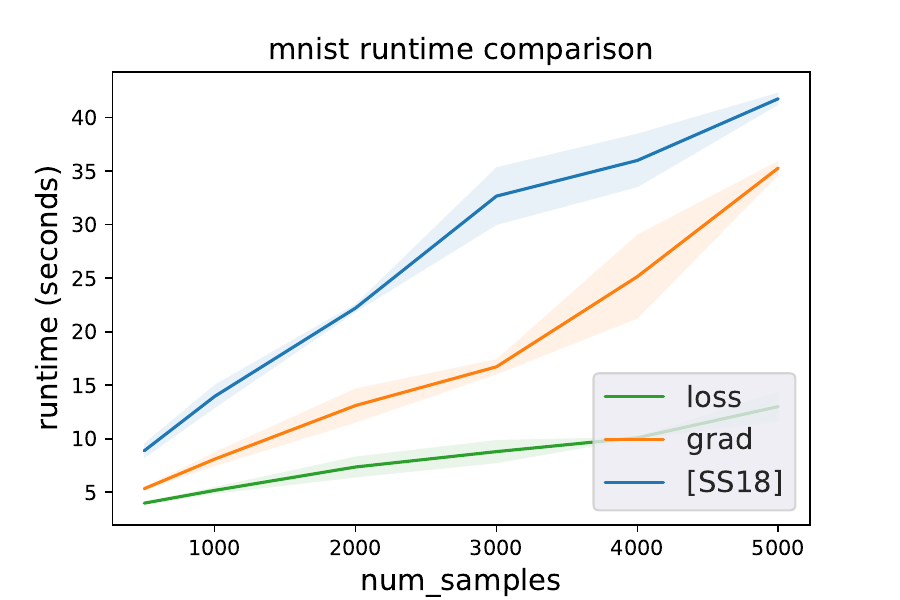}
\includegraphics[width=0.49\linewidth]{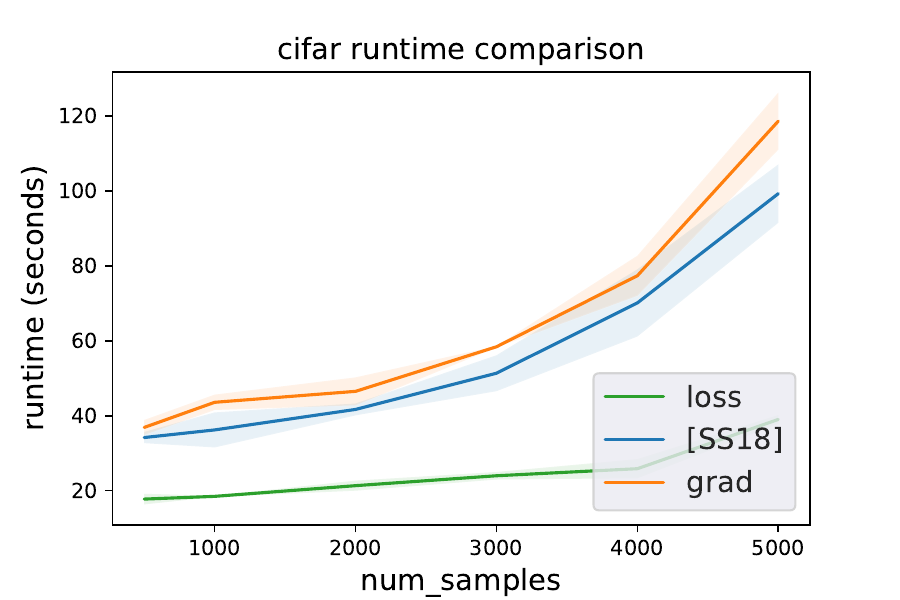}
\caption{We present runtime comparisons between different algorithms, for MNIST
and CIFAR10. The results for Fashion MNIST are analogous to those of MNIST, since
the model and dataset have the same size.}
\label{fig:runtimes}
\end{figure}

\paragraph{Algorithms for data selection.} We list some of the best-performing algorithms in literature, and mention their advantages and
disadvantages.
Out of these, margin and entropy sampling are the top performing methods in
the DISTIL\footnote{\url{https://github.com/decile-team/distil}} benchmark.%\todo{Q: If it is a benchmark, why don't we use it to compare our algorithms to the other? A: Their results are for active
%learning (not 1-round selections like ours). Also their 
%implementations are in pytorch while ours are in tensorflow.}
\begin{itemize}
\item Uniform sampling: We uniformly sample data points up to 
the budget $k$. This is the simplest and fastest way to sample $k$ data points.
\item {Margin/Least confidence/Entropy sampling: These methods
aim to select the examples with the lowest confidence.
Specifically, if $p_1,\dots,p_C$ are the per-class output 
probabilities of the model, we select the data points that either minimize $\max_{i\in [C]} p_i$, minimize $p_{i^*} - \max_{i\in [C]\backslash \{i^*\}} p_i$, where $i^* = \mathrm{argmax}_{i\in [C]} p_i$, or maximize the entropy $-\sum_{i=1}^C p_i\log p_i$.
Unfortunately, these methods require an inference call for 
\emph{each} data point, in order to evaluate its the 
classification uncertainty, and so are not considered runtime
efficient.
}
\item {$k$-center CoreSet [SS18]: The $k$-center algorithm from~\cite{SenerS18}. This
algorithm does not require any model inferences, but instead 
requires maintaining a nearest-neighbor data structure under
insertions of data points.}
\end{itemize}
%\todo{Q: What implementation are we using? All from distil? A: We use our own implementations}

In Figure~\ref{fig:results_appendix} we provide more detailed experiments,
including multiple algorithms
from previous work. 

\begin{figure}[h]
\begin{center}
%\framebox[4.0in]{$\;$}
%\fbox{\rule[-.5cm]{0cm}{4cm} \rule[-.5cm]{4cm}{0cm}}
\includegraphics[width=0.49\linewidth]{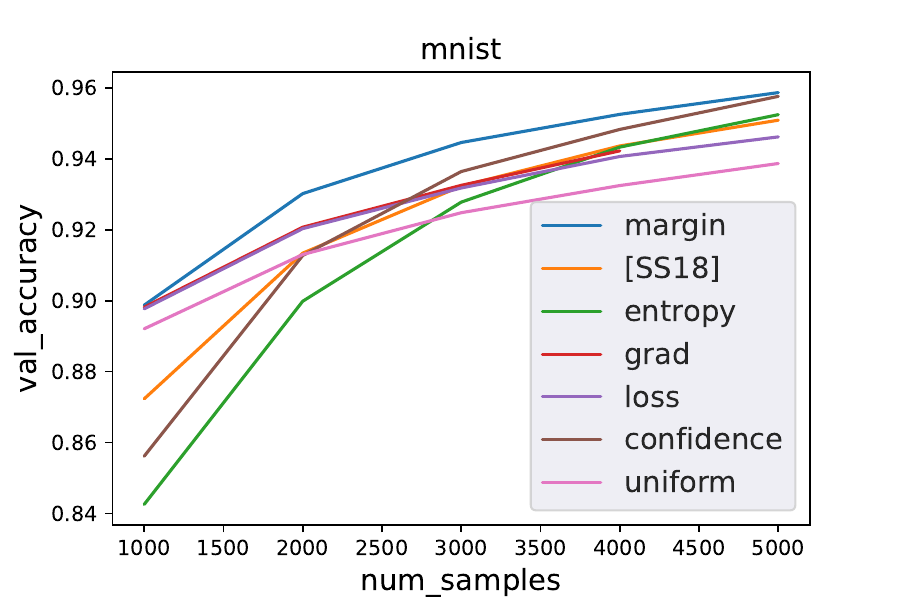}
\includegraphics[width=0.49\linewidth]{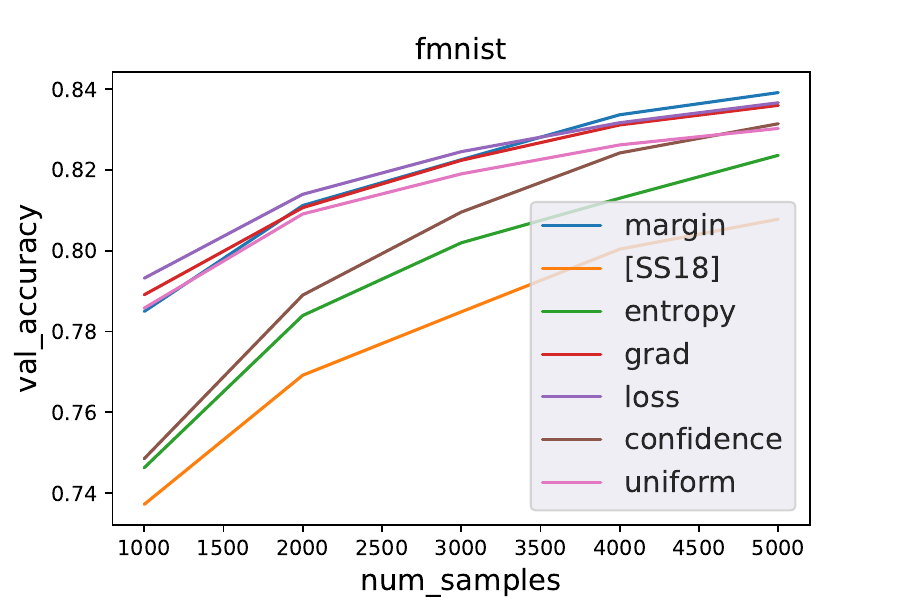}
\includegraphics[width=0.49\linewidth]{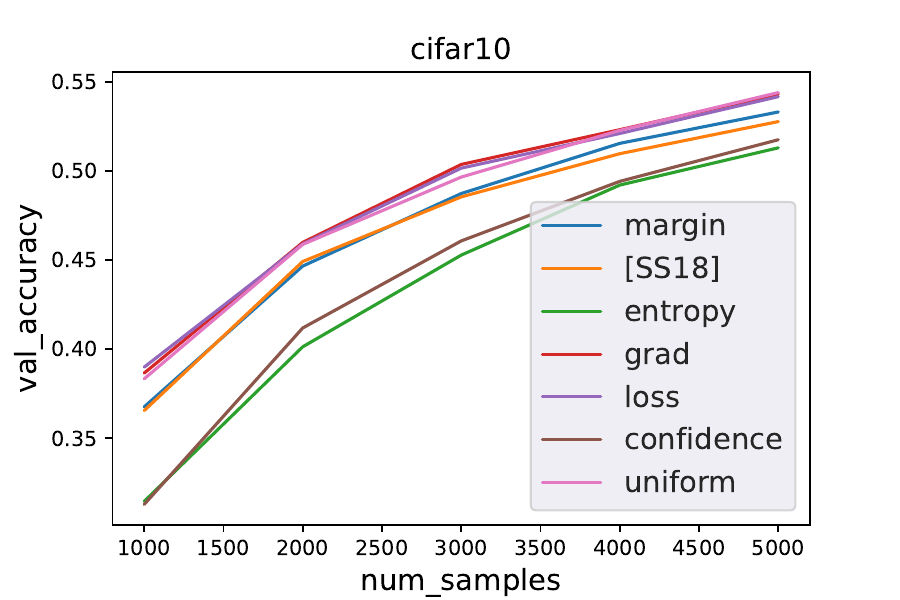}
\end{center}
\caption{The experimental results for different datasets and algorithms.
To minimize variation, we independently run each data point $100$ times.}
\label{fig:results_appendix}
\end{figure}

\subsection{Regression tasks}
\label{app:regression_algo_details}

\paragraph{Our algorithm.} We run \cref{alg:regression},
with a couple of differences: i) We run $k$-medoids, which is a
variant of $k$-median that is only allowed to pick input points as centers -- this
does not change the theoretical guarantees provided in the previous sections, 
and ii) we set $\Lambda_i\rightarrow\infty$ for all $i$,
which has the effect that we only look at distances and not losses.
We set the number of clusters to be $10\%$ of the total number of 
data points in the training set. After computing 
the regression solution, we evaluate it on the full training and 
validation datasets.

\end{document}